\documentclass{article}

\usepackage{microtype}
\usepackage{graphicx}
\usepackage{subfigure}
\usepackage{bm}
\usepackage{comment}
\usepackage[colorinlistoftodos]{todonotes}
\usepackage{amsmath,amssymb,amsthm,mathtools}
\usepackage[algo2e]{algorithm2e} 
\usepackage{booktabs} 
\usepackage{thm-restate}

\usepackage{hyperref}

\usepackage{color}

\usepackage{ccaption}
\usepackage{cite}

\newcommand{\argmax}{\mathop{\rm argmax}}
\newcommand{\argmin}{\mathop{\rm argmin}}

\renewcommand{\mid}{\,:\,}

\newcommand{\cE}{\mathcal{E}}

\newcommand{ \sm   }{ \scalebox{0.5} }

\newcommand{ \chiEofS}{\bm{\chi}_{ \scalebox{0.5} {$E(S)$} }}

\newcommand{ \chiEofSstar}{\bm{\chi}_{ \scalebox{0.5} {$E(S^*)$} }}
\newcommand{ \chiEofShat}{\bm{\chi}_{ \scalebox{0.5} {$E(\widehat{S}_t)$} }}
\newcommand{ \chiEofSout}{\bm{\chi}_{ \scalebox{0.5} {$E(S_{\tt OUT})$} }}

\newcommand{\Atin}{A_{{\bf x}_t}^{-1}}
\newcommand{\At}{A_{{\bf x}_t}}
\newcommand{\Atau}{A_{{\bf x}_{\tau}}}
  
\newcommand{\He}{H_{\epsilon} }

 \newtheorem{theorem}     {Theorem}
 \newtheorem{lemma}       {Lemma}
 \newtheorem{proposition} {Proposition}
 \newtheorem{problem}     {Problem}

\usepackage[accepted]{icml2020}

\icmltitlerunning{Online Dense Subgraph Discovery via Blurred-Graph Feedback}

\begin{document}

\twocolumn[
\icmltitle{
Online Dense Subgraph Discovery via Blurred-Graph Feedback}

\icmlsetsymbol{equal}{*}

\begin{icmlauthorlist}
\icmlauthor{Yuko Kuroki}{to,goo}
\icmlauthor{Atsushi Miyauchi}{to,goo}
\icmlauthor{Junya Honda}{to,goo}
\icmlauthor{Masashi Sugiyama}{goo,to}
\end{icmlauthorlist}

\icmlaffiliation{to}{The University of Tokyo, Japan}
\icmlaffiliation{goo}{RIKEN AIP, Japan}

\icmlcorrespondingauthor{Yuko Kuroki}{ykuroki@ms.k.u-tokyo.ac.jp}

\icmlkeywords{Densest subgraph problem, combinatorial multi-armed bandit, stochastic optimization, best arm identification}

\vskip 0.3in
]



\printAffiliationsAndNotice{}  

 \begin{abstract}
\emph{Dense subgraph discovery} aims to find a dense component in edge-weighted graphs. This is a fundamental graph-mining task with a variety of applications and thus has received much attention recently. Although most existing methods assume that each individual edge weight is easily obtained, such an assumption is not necessarily valid in practice. In this paper, we introduce a novel learning problem for dense subgraph discovery in which a learner queries edge subsets rather than only single edges and observes a noisy sum of edge weights in a queried subset. For this problem, we first propose a polynomial-time algorithm that obtains a nearly-optimal solution with high probability. Moreover, to deal with large-sized graphs, we design a more scalable algorithm with a theoretical guarantee. Computational experiments using real-world graphs demonstrate the effectiveness of our algorithms.
\end{abstract}


\section{Introduction}\label{sec:introduction}

{\em Dense subgraph discovery} aims to find a dense component in edge-weighted graphs. 
This is a fundamental graph-mining task with a variety of applications and thus has received much attention recently. 
Applications include detection of communities or span link farms in Web graphs~\cite{Dourisboure+07,Gibson+05}, 
molecular complexes extraction in protein--protein interaction networks~\cite{Bader_Hogue_03},
extracting experts in crowdsoucing systems~\cite{Kawase+19},
and real-time story identification in micro-blogging streams~\cite{Angel+12}. 

Among a lot of optimization problems arising in dense subgraph discovery, the most popular one would be the {\em densest subgraph problem}.
In this problem, given an edge-weighted undirected graph, we are asked to find a subset of vertices that maximizes the so-called \emph{degree density} (or simply \emph{density}), which is defined as half the average degree of the subgraph induced by the subset. 
Unlike most optimization problems for dense subgraph discovery, the densest subgraph problem can be solved exactly in polynomial time using some exact algorithms, e.g., Charikar's linear-programming-based (LP-based) algorithm~\cite{Charikar2000} and Goldberg's flow-based algorithm~\cite{Goldberg84}. 
Moreover, there is a simple greedy algorithm called the {\em greedy peeling}, which obtains a well-approximate solution in almost linear time~\cite{Charikar2000}.
Owing to the solvability and the usefulness of solutions, 
the densest subgraph problem has actively been studied in data mining, machine learning, and optimization communities~\cite{Ghaffari2019,Gionis_Tsourakakis_15,Miller2010,Papailiopoulos+14}. We thoroughly review the literature
in Appendix~\ref{sec:related}.

Although the densest subgraph problem requires a full input of the graph data, in many real-world applications, the edge weights need to be estimated from {\em uncertain} measurements.
For example, consider protein--protein interaction networks,
where vertices correspond to proteins in a cell and edges (resp. edge weights) represent the interactions (resp. the strength of interactions) among the proteins. 
In the generation process of such networks, the edge weights are estimated through biological experiments using measuring instruments with some noises~\cite{Nepusz+12}. 
As another example, consider social networks, where vertices correspond to users of some social networking service 
and edge weights represent the strength of communications (e.g., the number of messages exchanged) among them. 
In practice, we often need to estimate the edge weights by observing anonymized communications between users~\cite{Adar_Re_07}.

Recently, in order to handle the uncertainty of edge weights,
\citet{Miyauchi_Takeda_18} introduced a robust optimization variant of the densest subgraph problem.
In their method,
all edges are 
repeatedly
queried by a sampling oracle that returns an individual edge weight.
However, such a sampling procedure for individual edges is often quite costly or sometimes impossible.
On the other hand, it is often affordable to observe aggregated information of a subset of edges.
For example, in the case of protein--protein interaction networks, 
it may be costly to conduct experiments for all possible pairs of proteins, but it is cost-effective to observe molecular interaction among a molecular group ~\cite{Bader_Hogue_03}.
In the case of social networks, due to some privacy concerns and data usage agreements, 
it may be impossible even for data owners to obtain the estimated number of messages exchanged by two specific users, while it may be easy to access the information within some large group of users,
because this procedure reveals much less information of individual users~\cite{agrawal2000privacy,zheleva2011privacy}.

In this study, we introduce a novel learning problem for dense subgraph discovery, which we call {\em densest subgraph bandits} ({\em DS bandits}),
by incorporating the concepts of {\em stochastic combinatorial bandits}~\cite{chenwei13, Chen2014} into the densest subgraph problem.
In DS bandits, a learner is given an undirected graph, whose edge-weights are associated with unknown probability distributions.
During the exploration period,
the learner chooses a subset of edges (rather than only single edge) to sample,
and observes the sum of noisy edge weights in a queried subset; we refer to this  feedback model as {\em blurred-graph feedback}.
We investigate DS bandits with the objective of {\em best arm identification}, that is, the learner must report one subgraph that she believes to be optimal after the exploration period.

Our learning problem can be seen as a novel variant of \emph{combinatorial pure exploration} (CPE) problems~\cite{Chen2014,Chen2016matroid, Chen17a}.
In the literature, most existing work on CPE has considered the case where the learner obtains feedback from each arm in a pulled subset of arms, i.e., the {\em semi-bandit} setting, or each individual arm can be queried~(e.g.~\cite{Chen2014,Chen17a, bubeck2013, Gabillon2012, huang2018}).
Thus, the above studies cannot deal with the aggregated reward from a subset of arms.
On the other hand,
existing work on the {\em full-bandit} setting has assumed that the objective function is linear and the size of subsets to query is exactly $k$ at any round~\cite{Idan2019,Kuroki+19}, while our reward function (i.e., the degree density)
is not linear and the size of subsets to query is not fixed in advance. 
If we fix the size of subsets to query to $k$ in DS bandits,
the corresponding offline problem (called the densest $k$-subgraph problem) becomes  
NP-hard and the best known approximation ratio is just $\Omega(1/n^{{1/4}+\epsilon})$ for any $\epsilon>0$~\cite{Bhaskara+10}, where $n$ is the number of vertices.

The contribution of this work is three-fold and can be summarized as follows.

1)
We address a problem for dense subgraph discovery with no access to a sampling oracle for single edges (Problem~\ref{prob:graph}) in the  {\em fixed confidence setting}.
For this problem, we present a general learning algorithm \textsf{DS-Lin} (Algorithm~\ref{alg:main}) based on the technique of {\em linear bandits}~\cite{Auer03}.
We provide an upper bound of the number of samples that \textsf{DS-Lin} requires to identify an $\epsilon$-optimal solution with probability at least $1-\delta$ for $\epsilon>0$ and $\delta \in (0,1)$ (Theorem~\ref{thm:samplecomplexity}).
Our key idea is to utilize an approximation algorithm (Algorithm~\ref{alg:sdp}) to compute the maximal confidence bound,
thereby guaranteeing that the output by \textsf{DS-Lin} is an $\epsilon$-optimal solution and the running time is polynomial in the size of a given graph.


2)
To deal with large-sized graphs,
we further investigate another problem with access to sampling oracle for any subset of edges (Problem~\ref{prob:graph_fixbudget}) with a given fixed budget $T$.
For this problem,
we design a scalable and parameter-free algorithm \textsf{DS-SR} (Algorithm~\ref{alg:peeling_approxx2}) that runs in $O(n^2 T)$, while \textsf{DS-Lin} needs $O(m^2)$ time for updating the estimate, where $m$ is the number of edges.
Our key idea is to combine the {\em successive reject} strategy~\cite{Audibert2010} for the multi-armed bandits and the {\em greedy peeling} algorithm~\cite{Charikar2000} for the densest subgraph problem.
We prove an upper bound on the probability that
\textsf{DS-SR} outputs a solution whose degree density is less than $\frac{1}{2}{\rm OPT}-\epsilon$, where ${\rm OPT}$ is the optimal value (Theorem~\ref{thm:fixbudget}).
 

3)
In a series of experimental assessments, we thoroughly evaluate the performance of our proposed algorithms using well-known real-world graphs. 
We confirm that 
\textsf{DS-Lin} obtains a nearly-optimal solution even if the minimum size of queryable subsets is larger than the size of an optimal subset,
which is consistent with the theoretical analysis.
Moreover, we demonstrate that \textsf{DS-SR} finds nearly-optimal solutions even for large-sized instances,
while significantly reducing the number of samples for single edges required by
a state-of-the-art algorithm.

\section{Problem Statement}\label{sec:preliminaries}
In this section, we describe the densest subgraph problem and the online densest subgraph problem in the bandit setting formally.
\subsection{Densest subgraph problem}
The densest subgraph problem is defined as follows.
Let $G=(V,E,w)$ be an undirected graph, consisting of $n=|V|$ vertices and $m=|E|$ edges, with an edge weight $w:E\rightarrow \mathbb{R}_{>0}$, where $\mathbb{R}_{>0}$ is the set of positive reals. 
For a subset of vertices $S\subseteq V$, let $G[S]$ denote the subgraph induced by $S$, 
i.e., $G[S]=(S,E(S))$ where $E(S)=\{\{u,v\}\in E\mid u,v\in S\}$. 
The \emph{degree density} (or simply called the \emph{density}) of $S\subseteq V$ is defined as $f_w(S)=w(S)/|S|$, 
where $w(S)$ is the sum of edge weights of $G[S]$, 
i.e., $w(S)=\sum_{e\in E(S)}w(e)$. 
In the densest subgraph problem, given an edge-weighted undirected graph $G=(V,E,w)$, 
we are asked to find $S\subseteq V$ that maximizes the density $f_w(S)$.
There is an LP-based exact algorithm \cite{Charikar2000}, which is used in our proposed algorithm (see Appendix~\ref{sec:exact} for the entire procedure).

\subsection{Densest subgraph bandits (DS bandits)}
Here we formally define DS bandits.
Suppose that we are given an (unweighted) undirected graph $G=(V,E)$.
Assume that each edge $e \in E$ is associated with an unknown distribution $\phi_e$ over reals.
$w: E \rightarrow \mathbb{R}_{>0}$ is the expected edge weights, where $w(e)= \mathbb{E}_{X \sim \phi_e}[X]$.
Following the standard assumptions of stochastic multi-armed bandits, we assume that all edge-weight distributions have $R$-sub-Gaussian tails for some constant $R>0$.
Formally,
if $X$ is a random variable drawn from $\phi_e$ for $e \in E$,
then for all $r \in \mathbb{R} $, $X$ satisfies $\mathbb{E}[{\exp (rX -r\mathbb{E}[X]) }] \leq \exp(R^2r^2/2)$.
We define the optimal solution as $S^*=\argmax_{S \subseteq V} f_w(S)$.


We first address the setting in which the learner can stop the game at any round if she can return an $\epsilon$-optimal solution for $\epsilon >0$ with high probability. 
Let $k >2$ be the minimal size of queryable subsets of vertices; 
notice that the learner has no access to a sampling oracle for single edges.
The problem is formally defined below.

\begin{problem}[DS bandits with no access to single edges]\label{prob:graph}
We are given an undirected graph $G=(V,E)$ and a family of queryable subsets of at least $k\ (>2)$ vertices $\mathcal{S}\subseteq 2^{V}$.
Let $\epsilon >0$ be a required accuracy and $\delta \in (0,1)$ be a confidence level.
Then, the goal is to find $S_{\tt OUT} \subseteq V$ that satisfies $\Pr [f_w(S^*)-f_w(S_{\tt OUT}) \leq \epsilon] \geq 1-\delta$,
while minimizing the number of samples required by an algorithm (a.k.a.~the sample complexity).  
\end{problem}

We next consider the setting in which the number of rounds in the exploration
phase is fixed and is known to the learner,
and the objective is to maximize the quality of the output solution.
In this setting, we relax the condition of queryable subsets; 
assume that the learner is allowed to query any subset of edges.
The problem is defined as follows.

\begin{problem}[DS bandits with a fixed budget]\label{prob:graph_fixbudget}
We are given an undirected graph $G=(V,E)$ and a fixed budget $T$. 
The goal is to find $S_{\tt OUT}\subseteq V$ that maximizes $f_w(S_{\tt OUT})$ within $T$ rounds.
\end{problem}

Note that Problem~\ref{prob:graph} is often called the \emph{fixed confidence setting} and Problem~\ref{prob:graph_fixbudget} is called the \emph{fixed budget setting} in the bandit literature.

 \section{Algorithm for Problem~\ref{prob:graph}}\label{sec:fixed_confidence}

In this section, we first present an algorithm for Problem~\ref{prob:graph} based on linear bandits, which we refer to as \textsf{DS-Lin}.
We then show that \textsf{DS-Lin} is {\em $(\epsilon,\delta)$-PAC}, that is, the output of the algorithm satisfies $\Pr [f_w(S^*)-f_w(S_{\tt OUT})  \leq \epsilon] \geq 1-\delta$.
Finally, we provide an upper bound of the number of samples (i.e.,~the sample complexity).

\subsection{\textsf{DS-Lin} algorithm}
We first explain how to obtain the estimate of edge weights and confidence bounds. Then we discuss how to ensure a stopping condition and describe the entire procedure of \textsf{DS-Lin}. 

\paragraph{Least-squares estimator.}
We construct an estimate of edge weight $w$ using a sequential noisy observation. 
For $S\subseteq V$,
let $\chiEofS \in \{0,1\}^E$ be the indicator vector of $E(S) \subseteq  E$,
i.e., for each $e \in E$, $\chiEofS(e)=1$ if $e \in E(S)$ and $\chiEofS(e)=0$ otherwise.
Therefore, 
each subset of edges $E(S)$ for $S\subseteq V$ corresponds to an arm whose feature is an indicator vector of it in linear bandits. 
For any $t > m$,
we define a sequence of indicator vectors as 
${\bf x}_t=(\bm{\chi}_{\sm {$E(S_1)$}},\ldots,\bm{\chi}_{\sm {$E(S_t)$}}) \in \{0,1\}^{E \times t}$ and also define the corresponding sequence of observed rewards as $(r_1(S_1), \ldots, r_t(S_t)) \in \mathbb{R}^{ t}$.
We define $\At$ as 
\[
A_{{\bf x}_t}=\sum_{i=1}^t \bm{\chi}_{\sm {$E(S_i)$}} \bm{\chi}_{\sm {$E(S_i)$}}^{\top} + \lambda I   \in \mathbb{R}^{E \times E}
\]
for a regularized term $\lambda >0$, where $I$ is the identity matrix.
Let $b_{{\bf x}_t}=\sum_{i=1}^t \bm{\chi}_{\sm {$E(S_i)$}}r_i(S_i) \in \mathbb{R}^E.$
Then, the {\em regularized least-squares estimator} for $w \in \mathbb{R}^E$
can be obtained by
\begin{align}\label{eq:estimate}
\widehat{w}_t
=A_{{\bf x}_t }^{-1}b_{{\bf x}_t}\in \mathbb{R}^{E}.
\end{align}

\paragraph{Confidence bounds.}
The basic idea to deal with uncertainty is that we maintain {\em confidence bounds} that contain the parameter $w \in \mathbb{R}^E$ with high probability.
For a vector $x \in \mathbb{R}^m $ and a matrix $B\in \mathbb{R}^{m\times m}$, let $\|x\|_B=\sqrt{x^\top Bx}$.
Let $ N(v) =\{ u \in V \mid \{ u, v\} \in E\}$ be the set of neighbors of $v\in V$ and $ \text{deg}_{\max} =\max_{v \in V}|N(v)|$ be the maximum degree of vertices.
In the literature of linear bandits, \citet{yadkori11} proposed a high probability bound on confidence ellipsoids with a center at the estimate of unknown expected rewards.
Plugging it into our setting, 
we have the following proposition on the ellipsoid confidence bounds for the estimate $\widehat{w}_t=A_{{\bf x}_t }^{-1}b_{{\bf x}_t}$, where ${\bf x}_t$ is fixed beforehand: 
\begin{proposition}[Adapted from \citet{yadkori11}, Theorem~2]\label{propo:yadokori}
Let $\eta_t$ be an R-sub-Gaussian noise for $R>0$ and $R'=\sqrt{ \mathrm{deg}_{\max}}R$.
Let $\delta \in (0,1)$ and assume that the $\ell_2$-norm of edge weight $w$ is less than $L$.
Then, for any fixed sequence ${\bf x}_t$,
with probability at least $1-\delta$,
the inequality
\begin{align}\label{ineq:ellipsoid}
\left| w(S)-\widehat{w}_t(S) \right| \leq C_t \| \chiEofS  \|_{\Atin}
\end{align}
holds for all $t \in \{ 1,2,\ldots\}$
and all $S \subseteq V$,
where
\begin{align}\label{confideincebound}
    C_t = R'\sqrt{2 \log 
\frac{\det(\At)^{\frac{1}{2}}}
{\lambda^{\frac{m}{2}} \delta}} +\lambda^{\frac{1}{2}} L.
\end{align}
\end{proposition}

The above bound can be used to guarantee the accuracy of the estimate.

\paragraph{Computing the maximal confidence bound.}
To identify a solution with an optimality guarantee,
the learner ensures whether the estimate is valid by computing the maximal confidence bound among all subsets of vertices. 
We consider the following stopping condition: 
\begin{alignat*}{4}
& f_{\widehat{w}_t}(\widehat{S}_t)- \frac{C_t \|\chiEofShat \|_{\Atin}}{|\widehat{S}_t|}  \\
& \geq \max_{  S \subseteq  V \mid S \neq \widehat{S}_t }f_{\widehat{w}_t}(S)+ \frac{C_t \max_{S \subseteq V} \|\chiEofS \|_{\Atin} }{|S|} -\epsilon.
\end{alignat*}
The above stopping condition guarantees that the output satisfies $f_w(S^*)-f_w(S_{\tt OUT})  \leq \epsilon$ with probability at least $1-\delta$.
However, computing $\max_{S \subseteq V} \|\chiEofS \|_{\Atin}$ by brute force is intractable since it involves an exponential blow-up in the number of $S \subseteq V$.
To overcome this computational challenge,
we address a relaxed quadratic program:
\begin{alignat}{4}\label{prob:CBM} 
& &\ \text{P1:} \ &\text{max.} & \  & \| x \|_{\Atin} 
\   \       \text{s.t.}  \ -e \leq x \leq e,
\end{alignat}
where $e \in \mathbb{R}^m$ is the vector of all ones.

There is an efficient way to solve $\text{P1}$ using the SDP-based algorithm~by \citet{Ye97approximatingquadratic} for the following quadratic program with bound constraints:
\begin{alignat}{4}\label{prob:qp1} 
& &\ \text{QP:} \ &\text{max.} & \  &  \sum_{1 \leq i , j \leq m}q_{ij}x_ix_j 
\   \       \text{s.t.}  \ -e \leq x \leq e,
\end{alignat}
where $Q=(q_{ij}) \in \mathbb{R}^{m \times m}$ is a given symmetric matrix.
\citet{Ye97approximatingquadratic} modified the algorithm by \citet{GoemansWilliamson1995} and generalized the proof technique of \citet{Nestrov1998}, and then established the constant-factor approximation result for QP.
\begin{proposition}[\citet{Ye97approximatingquadratic}]
There exists a polynomial-time $\frac{4}{7}$-approximation algorithm for QP.
\end{proposition}
Note that Ye's algorithm~\cite{Ye97approximatingquadratic} is a randomized algorithm, but it can be derandomized using the technique devised by~\citet{mahajan1999}.
The learner can compute an upper bound of the maximal confidence bound $\max_{S \subseteq V} \| \chiEofS \|_{\Atin}$ by using an approximate solution to QP obtained by the derandomized version of Ye's algorithm, because it is obvious that the optimal value of QP is larger than $\max_{S \subseteq V} \| \chiEofS \|^2_{\Atin}$.
Therefore,
using Algorithm~\ref{alg:sdp}, we can ensure the following stopping condition in polynomial time:
\begin{alignat}{2}\label{stop}
&f_{\widehat{w}_t}(\widehat{S}_t)- \frac{C_t \| \chiEofShat \|_{\Atin}}{|\widehat{S}_t|} \notag \\
& \geq \max_{  S \subseteq  V \mid S \neq \widehat{S}_t }f_{\widehat{w}_t}(S)+ \frac{C_t Z_t}{2 \alpha } -\epsilon,
\end{alignat}
where $Z_t$ denotes the objective value of the approximate solution to $\text{P1}$ and $\alpha$ is a constant-factor approximation ratio of Algorithm~\ref{alg:sdp}.

\begin{algorithm}[t]
\caption{ Unconstrained 0--1 quadratic programming}\label{alg:sdp}
 	\SetKwInOut{Input}{Input}
	\SetKwInOut{Output}{Output}
	\Input{ A positive semidefinite matrix $Q \in \mathbb{R}^{m \times m} $ }
	\Output{ $x \in [-1,1]^m$}

    Solve the following quadratic programming problem by Ye's algorithm~\cite{Ye97approximatingquadratic} with derandomization \cite{mahajan1999}: 
\begin{alignat*}{4}
& &\ \text{QP:} \ &\text{max.} & \  &  \sum_{1 \leq i , j \leq m}q_{ij}x_ix_j
\   \       \text{s.t.}  \ -e \leq x \leq e,
\end{alignat*}
and obtain a solution $\bar{x} \in [-1,1]^m$; 

    \Return{ $\bar{x}$
    }
\end{algorithm}

\begin{algorithm}[t]
\caption{\textsf{DS-Lin}}
\label{alg:main}
	\SetKwInOut{Input}{Input}
	\SetKwInOut{Output}{Output}
	\Input{ Graph $G=(V,E)$,
	a family of queryable subsets of at least $k\ (>2)$ vertices  $\mathcal{S} \subseteq  2^V $,
	parameter $\epsilon>0$, parameter $\delta \in (0,1)$, and allocation strategy $p$}
	\Output{ $S \subseteq V$}
	
    \For{$t =1, \ldots, m$}{
    
    Choose $S_t \leftarrow \argmin_{S \in {\rm supp}(p)} \frac{T_t(S)}{p(S)}$;
    
    Call the sampling oracle for $S_t$;

    Observe $r_t(S_t)$;
    
    
    
    $b_{{\bf x}_t} \leftarrow b_{{\bf x}_{t-1}}+\bm{\chi}_{\sm {$E(S_t)$}} r_t(S_t)$;
    }
	\While{stopping condition~\eqref{stop} is not true}{
    $t \leftarrow t+1$;
    
  
    Choose $S_t \leftarrow \argmin_{S \in {\rm supp}(p)} \frac{T_t(S)}{p(S)}$;
    
    Call the sampling oracle for $S_t$ and observe $r_t(S_t)$;
    
    
    
    $A_{{\bf x}_t} \leftarrow A_{{\bf x}_{t-1}}+ \bm{\chi}_{\sm {$E(S_t)$}} \bm{\chi}^{\top}_{\sm {$E(S_t)$}}$;
    
    $b_{{\bf x}_t} \leftarrow b_{{\bf x}_{t-1}}+\bm{\chi}_{\sm {$E(S_t)$}} r_{S_t}$;
    
    $\widehat{w}_t \leftarrow A_{\bf{x}_t}^{-1} b_t$;
    
    \textbf{If} $\widehat{w}_t(e)<0$ \textbf{then} $\widehat{w}_t(e)=0$ for each $e \in E$;
    
    $x \leftarrow$ Algorithm~\ref{alg:sdp} for $\Atin$;
    
    $Z_t \leftarrow C_t \sqrt{\sum_{1 \leq i , j \leq m}\Atin(i,j)x_ix_j}$;
    
    $\widehat{S}_t \leftarrow $ Output of the LP-based exact algorithm~\cite{Charikar2000} for $G(V,E,\widehat{w}_t)$;
    
    }

    
    \Return{$S_{\tt OUT} \leftarrow \widehat{S}_t$}
\end{algorithm}

\paragraph{Proposed algorithm.}
Let $T_t(S)$ be the number of times that $S \subseteq {\cal S}$ is queried before $t$-th round in the algorithm.
We present our algorithm \textsf{DS-Lin}, which is detailed in Algorithm~\ref{alg:main}.
Our sampling strategy is based on a given allocation strategy $p$ defined as follows.
Let ${\cal P}$ be a $|{\cal S} |$-dimensional probability simplex. We define $p$ as $p=(p(S))_{S \in {\cal S}} \in {\cal P}$, where $p(S)$ describes the predetermined proportions of queries to a subset $S$.
As a possible strategy $p$, one can use the well-designed strategy called $G$-allocation~\cite{pukelsheim2006,Soare2014}, or simply use uniform allocation (see Appendix~\ref{apx:allocation} for details).
At each round $t$,
the algorithm calls the sampling oracle for $S_t \in \mathcal{S}$ and  observes $r_t(S_t)$.
Then, the algorithm updates statistics $\At$ and $b_{{\bf x}_t}$, and also updates the estimate $\widehat{w}_t$.
To check the stopping condition, the algorithm approximately solves P1 by Algorithm~\ref{alg:sdp} and computes the empirical best solution $S_t$ using the LP-based exact algorithm for the densest subgraph problem for $G=(V,E,\widehat{w}_t)$. 
Once the stopping condition is satisfied, 
the algorithm returns the empirical best solution $S_t$ as output.

 \subsection{Sample complexity}\label{sec:analysis}
We prove that \textsf{DS-Lin} is $(\epsilon,\delta)$-PAC and analyze its sample complexity.
We define the design matrix for $p \in {\cal P}$ as  $\Lambda_p = \sum_{S \in {\cal S}} p(S) \chiEofS \chiEofS^{\top}$.
We define $\rho_{\Lambda_p}$  as $\rho_{\Lambda_p}= \max_{x \in [-1,1]^m} \|x \|^2_{{\Lambda_p}^{-1}}$.
Let $\Delta_{\min}$ be the minimal gap between the optimal value and the second optimal value, i.e., $\Delta_{\min}= \min_{S \subseteq V \mid S \neq S^*} f_w(S^*)- f_w(S)$.
The next theorem shows an upper bound of the number of queries required by Algorithm~\ref{alg:main} to output $S_{\tt OUT}\subseteq V$ that satisfies $\Pr[f_w(S^*)-f_w(S_{\tt OUT}) \leq \varepsilon] \geq 1-\delta$.

\begin{theorem}\label{thm:samplecomplexity}
Define $\He= \frac{ \rho_{\Lambda_p}+\epsilon }{ (\Delta_{\min} + \epsilon)^2}$.
Then, with probability at least $1- \delta$, 
\textsf{DS-Lin} (Algorithm~\ref{alg:main}) outputs $S\subseteq V$ whose density is at least $f_w(S^*)-\epsilon$ and the total number of samples $\tau$ is bounded as follows:

if $\lambda>4 m(\sqrt{m}+\sqrt{2})^2  \mathrm{deg}_{\max}R^2 \He$, then
\begin{align*}
\tau=O \left(\left(\mathrm{deg}^2_{\max} R^2 \log \frac{1}{\delta}+ \lambda L^2 \right)  \He \right),
\end{align*}
and
if $\lambda \leq  \frac{\mathrm{deg}_{\max}R^2}{L^2} \log \left(\frac{1}{\delta}\right)$,
then
\begin{align*}
\tau =  O \left( m \mathrm{deg}_{\max} R^2 \He \log \frac{1}{\delta}+C_{\He, \delta} \right) 
\end{align*}
where  $C_{\He, \delta}$ is 
\begin{align*}
O\left( m\mathrm{deg}_{\max}  R^2 \He  \log \left(\mathrm{deg}_{\max} R m \He \log \frac{1}{\delta} \right) \right).
\end{align*}
\end{theorem}
The proof of Theorem~\ref{thm:fixbudget} is given in Appenfix~\ref{apx:proof_theorem1}.
Note that $\rho_{\Lambda_p}=d$ holds if we are allowed to query any subset of vertices and employ G-allocation strategy, i.e., $p=\argmin_{p \in \mathcal{P}} \max_{S \subseteq V}  \|\chiEofS \|^2_{{\Lambda_p}^{-1}}$, which was shown in ~\citet{kiefer_wolfowitz_1960}.
However, in practice, we should restrict the size of the support to reduce the computational cost; finding a family of subsets of vertices that minimizes $\rho_{\Lambda_p}$ may be also related to the optimal experimental design problem~\cite{pukelsheim2006}.


In the work of~\citet{Chen2014},
they proved that the lower bound on the sample complexity of general combinatorial pure exploration problems with linear rewards is $\Omega(\sum_{e \in [m]} \frac{1}{\Delta_{e}^2} \log \frac{1}{\delta})$, where $m$ is the number of base arms and $\Delta_e$ is defined as follows.
Let ${\cal M}$ be any decision class (such as size-$k$, paths, matchings, and matroids).
Let $M^*$ be an optimal subset, i.e., $M^*=\argmax_{M \in {\cal M}}\sum_{e \in M}w_e$.
For each base arm $e \in [m]$,
the gap $\Delta_e$ is defined as
 $ \Delta_e = \sum_{e \in M^*}w_e -\max_{M \in {\cal M} \mid e \in M} \sum_{e \in M}w_e \  ({\rm if} \ e \notin M^* )$, and 
 $ \Delta_e = \sum_{e \in M^*}w_e-\max_{M \in {\cal M} \mid e \notin M} \sum_{e \in M}w_e \   ({\rm if} \ e \in M^* )$.

In the work of ~\citet{huang2018},
they studied the combinatorial pure exploration problem with continuous and separable reward functions, and showed that the problem has a lower bound $\Omega({\bf H_{\Lambda}} +{ \bf H_{\Lambda} } m^{-1} \log (  \delta^{-1}))$, where ${ \bf H_{\Lambda} }= \sum_{i=1^m} \frac{1}{\Lambda^2_i}$.
In their definition of ${ \bf H_{\Lambda} }$,
the term $\Lambda_i$ is called {\em consistent optimality radius} and it measures how far the estimate can be away from true parameter while the optimal solution in terms of the estimate is still consistent with the true optimal one in the $i$-th dimension (see Definition~2 in~\cite{huang2018}).

Note that the problem settings in~\citet{Chen2014} and ~\citet{huang2018} are different from ours; in fact, in our setting the learner can query a subset of edges rather than a base arm and reward function is not linear.
Therefore, their lower bound results are not directly applicable to our problem.
However, we can see that our sample complexity in Theorem~\ref{thm:samplecomplexity} is comparable with their lower bounds because ours is $O(\He \log \delta^{-1}+ \He \log(\He \log \delta^{-1}))$ if we ignore the terms irrespective of $\He$ and $\delta$.

\section{Algorithm for Problem~\ref{prob:graph_fixbudget}}
In this section, we propose a scalable and parameter-free algorithm for Problem~\ref{prob:graph_fixbudget}
that runs in $O(n^2 T)$ time for a given budget $T$, and provide theoretical guarantees for the output of the algorithm.
\subsection{\textsf{DS-SR} algorithm}
The design of our algorithm is based on the Successive
Reject (SR) algorithm, which was designed for a regular multi-armed bandits in the fixed budget setting~\cite{Audibert2010} and is known to be the optimal strategy~\cite{capentier16}.
In classical SR algorithm, we divide the budget $T$ into $K-1$ ($K$ is the number of arms) phases.
During each phase, the algorithm uniformly samples an {\em active} arm that has not been dismissed yet.
At the end of each phase, the algorithm dismisses the arm with the lowest empirical mean.
After $K$ phases, the algorithm outputs the last surviving arm.

For DS bandits,
we employ a different strategy from the classical one because our aim is to find the best subset of vertices in a given graph.
Specifically, our algorithm \textsf{DS-SR} is inspired by
the graph algorithm called {\em greedy peeling}~\cite{Charikar2000}, which was designed for approximately solving the densest subgraph problem.
\textsf{DS-SR} removes one vertex in each phase, and after all phases are over, it selects the best subset of vertices according to the empirical observation. 

\paragraph{Notation.}
For $S \subseteq V$ and $ v \in S$, let $N_S(v) =\{u \in S \mid \{u,v\} \in E \}$ be the set of neighboring vertices of $v$ in $G[S]$ and let $E_S(v) = \{\{u, v\}\in E \mid u\in N_S(v) \}$ be the set of incident edges to $v$ in $G[S]$.
For $F \subseteq 2^E$  and for all phases $t \geq 1$,
we denote by $T_{F}(t)$ the number of times that $F$ was sampled over all rounds from 1 to $t$,
and denote by $X_{F}(1) , \ldots, X_{F}(T_{F}(t))$ the sequence of associated observed weights.
Introduce $\hat{X}_{F}(k)= \frac{1}{k} \sum_{s=1}^k X_{F}(s)$ as the empirical mean of weights of $F$ after $k$ samples.
 For simplicity, we denote $\widehat{ \mathrm{deg}}_{S,v}(t)= \hat{X}_{E_{S}(v)} \left( T_{E(S)}(t) \right)$.

\begin{algorithm}[t]
\caption{ \textsf{DS-SR}}
\label{alg:peeling_approxx2}
 	\SetKwInOut{Input}{Input}
 	\SetKwInOut{Output}{Output}
	\Input{ Budget $T>0$, graph $G(V,E)$, sampling oracle }
	\Output{ $S \subseteq V$}
	
	$\tilde{\log}(n-1) \leftarrow \sum_{i=1}^{n-1}\frac{1}{i}$;
	
	$\tilde{T}_0 \leftarrow 0$;
	
	For $T_0(v) \leftarrow 0$ for each $v \in V$;
	
	
	$S_n \leftarrow V $ and	$v_0 \leftarrow \emptyset $;
	
    \For{$t \leftarrow 1, \ldots, n-1$}{

         $\tilde{T_t} \leftarrow  \left \lceil \frac{T-\sum_{i=1}^{n+1}i}{\tilde{\log}(n-1)(n-t) }\right \rceil$;
         
        $T'_{t} \leftarrow \left\lceil  \frac{\tilde{T_t}}{2|S_{n-t+1}|} \right \rceil $ and $\tau_{t} \leftarrow T'_t- T'_{t-1}$;
               

        
  

    
         
         \For{ $v \in S_{n-t+1}$}{
         Run Algorithm~\ref{alg:sampling} (sampling procedure);
         }
         
         $\widehat{f}(S_{n-t+1}) \leftarrow \frac{\frac{1}{2}\sum_{v \in S_{n-t+1} \widehat{ \mathrm{deg}}_{S_{n-t+1}}(v,t)} }{|S_{n-t+1}|}$;         

         $v_t \leftarrow  \argmin_{v \in S_{n-t+1}} \widehat{ \mathrm{deg}}_{S_{n-t+1}}(v,t)$;

    $S_{n-t} \leftarrow S_{n-t+1 }\setminus \{v_t\}$;

    }

    \Return{ $S_{{\tt OUT}} \in \{ S_2, \ldots, S_n \} $ that maximizes $\widehat{f}(S_i)$
    }
\end{algorithm}

\begin{algorithm}[t]
\caption{Sampling procedure (subroutine of Algorithm~\ref{alg:peeling_approxx2})}
\label{alg:sampling}
	\SetKwInOut{Input}{Input}
	\SetKwInOut{Output}{Output}

        \If{ $N_{S_{n-t+1}}(v)=\emptyset$}{Set $\widehat{\mathrm{deg}}_{S_{n-t+1}}(v,t)=0$;}
        
        \Else{
         

                 
                 
             \If{ $v \notin N_{S_{n-t+2}}(v_{t-1})$}
             { 
             
            Sample $E_{S_{n-t+1}}(v)$ for $\tau_{t}$ times;
             
             $Y_t \leftarrow T_{E_{S_{n-t+1}}(v)}(t-1)  \widehat{\mathrm{deg}}_ {S_{n-t+2}}(v,t) $;
                 
            $\widehat{\mathrm{deg}}_{S_{n-t+1}}(v,t)\leftarrow  \frac{Y_t+\tau_t \hat{X}_{ E_{S_{n-t+1}}(v)}(\tau_t)}{T_{E_{S_{n-t+1}}(v)}(t-1) +\tau_t}$;
            
            $T_{E_{S_{n-t+1}}(v)}(t) \leftarrow T_{E_{S_{n-t+1}}(v)}(t-1) +\tau_t$;
            }

             \Else{
             
            Sample $E_{S_{n-t+1}}(v)$ for $\sum_{i=1}^t \tau_i$ times;
          
             $  \widehat{ \mathrm{deg}}_{S_{n-t+1}}(v,t) \leftarrow 
           \hat{X}_{ E_{S_{n-t+1}}}(v) (\sum_{i=1}^t \tau_i)$;

             $T_{E_{S_{n-t+1}}(v)}(t) \leftarrow \sum_{i=1}^t \tau_{i}$;
             


            }

        }


\end{algorithm}


\paragraph{Proposed algorithm.}
All procedures of \textsf{DS-SR} are detailed in Algorithm~\ref{alg:peeling_approxx2}. Intuitively, \textsf{DS-SR} proceeds as follows.
Given a budget $T$, we divide $T$ into $n-1$ phases.
\textsf{DS-SR} maintains a subset of vertices.
Initially $S_n \leftarrow V$.
In each phase $t$, for $v \in S_{n-t+1}$, the algorithm uses the sampling oracle for obtaining the estimate of the degree $\widehat{ \mathrm{deg}}_{S_{n-t+1}}(v)$, which we refer to as the {\em empirical degree}.
After the sampling procedure, we compute {\em empirical quality function} $ \widehat{f}(S_{n-t+1})$ and specify one vertex $v_t$ that should be removed.
In Algorithm~\ref{alg:sampling}, we detail the sampling procedure for obtaining the empirical degree of $v \in S_{n-t+1}$.
If $v$ was not a neighbor of $v_{t-1}$ in phase $t-1$,
the algorithm samples $E_{S_{n-t+1}}(v)$ for $\tau_t$ times, where $\tau_t$ is set carefully.
On the other hand,
if $v$ was a neighbor of that,
the algorithm samples $E_{S_{n-t+1}}(v)$ for $\sum_{i=1}^t \tau_i$ times.
Our eliminate scheme removes a vertex $v_t$ that minimizes the empirical degree, i.e.,
$v_t \in \argmin_{v \in S_{n-t+1}} \widehat{ \mathrm{deg}}_{S_{n-t+1}}(v,t)$.
Finally, after $n-1$ phases have been done, 
\textsf{DS-SR} outputs $S_{\tt OUT}\subseteq V$ that maximizes the empirical quality function, i.e.,  $S_{\tt OUT}=\argmax_{S_i \in \{ S_2, \ldots, S_n \}} \widehat{f}(S_i)$.

\subsection{Upper bound on the probability of error}
We provide an upper bound on the probability that the quality of solution obtained by the proposed algorithm is less than $\frac{1}{2}f_w(S^*)-\epsilon$, as shown in the following theorem.

\begin{theorem}\label{thm:fixbudget}
Given any $T>m$, and assume that the edge weight distribution $\phi_e$ for each arm $e \in [m]$ has mean $w(e)$ with an $R$-sub-Gaussian tail.
Then, \textsf{DS-SR} (Algorithm~\ref{alg:peeling_approxx2}) uses at most $T$ samples and outputs $S_{\tt OUT} \subseteq V$ such that
\begin{alignat}{4}\label{proberror}
  &  \Pr\left[f_w(S_{\tt OUT}) < \frac{f_w(S^*)}{2} - \epsilon \right]
  \notag \\
  & \leq C_{G,\epsilon} \exp \left( - \frac{(T-\sum_{i=1}^{n+1}i) \epsilon^2 }{4 n^2 \mathrm{deg}_{\max} R^2  \tilde{ \log }(n-1) } \right),
\end{alignat}
where $C_{G,\epsilon}=\frac{2 \mathrm{deg}_{\max}(n+1)^3 2^n R^2}{\epsilon^2}$ and $\tilde{\log}(n-1) = \sum_{i=1}^{n-1}  i^{-1}$.
\end{theorem}
The proof of Theorem~\ref{thm:fixbudget} is given in Appenfix~\ref{apx:proof_theorem2}.
From the theorem,
we see that \textsf{DS-SR} requires a budget of $T=O\left (\frac{n^3 \mathrm{deg}_{\max}}{\epsilon^2} \log \left( \frac{\mathrm{deg}_{\max}}{\epsilon}\right) \right)$
by setting the RHS of~\eqref{proberror} to a constant.
Besides,
the upper bound on the probability of error is exponentially decreasing with $T$.

\section{Experiments}\label{sec:experiments}

In this section, we examine the performance of our proposed algorithms \textsf{DS-Lin} and \textsf{DS-SR}.
First, we conduct experiments for \textsf{DS-Lin} and show that \textsf{DS-Lin} can find a nearly-optimal solution without sampling any single edges.
Second, we perform experiments for \textsf{DS-SR} and demonstrate that \textsf{DS-SR} is applicable to large-sized graphs and significantly reduces the number of samples for single edges, compared to that of the state-of-the-art algorithm. 
Throughout our experiments,
to solve the LPs in Charikar's algorithm~\cite{Charikar2000}, we used a state-of-the-art mathematical programming solver, 
Gurobi Optimizer 7.5.1, with default parameter settings. 
All experiments were conducted on a Linux machine with 2.6~GHz CPU and 130~GB RAM. 
The code was written in Python.

\paragraph{Dataset.} Table~\ref{tab:instance} lists real-world graphs on which our experiments were conducted. 
Most of those can be found on Mark Newman's website\footnote{http://www-personal.umich.edu/~mejn/netdata/} 
or in SNAP datasets\footnote{http://snap.stanford.edu/}. 
For each graph, we construct the edge weight $w$ using the following simple rule, 
which is inspired by the {\em knockout densest subgraph model} introduced by \citet{Miyauchi_Takeda_18}. 
Let $G=(V,E)$ be an unweighted graph and let $S^*\subseteq V$ be an optimal solution to the densest subgraph problem. 
For each $e\in E$, we set $w(e)=\texttt{rand}(1,20)$ if $e\in E(S^*)$, 
and $w(e)=\texttt{rand}(1,100)$ if $e\in E\setminus E(S^*)$, 
where $\texttt{rand}(\cdot, \cdot)$ is the function that returns a real value selected uniformly at random from the interval between the two values. 
That is, we set a relatively small value for each $e\in E(S^*)$ and a relatively large value for each $e\in E\setminus E(S^*)$, 
which often makes the densest subgraph on $G=(V,E)$ no longer densest on the edge-weighted graph $G=(V,E,w)$. 
Throughout our experiments,
we generate a random noise $\eta(e) \sim \mathcal{N}(0,1)$ for all $e \in E$.

\begin{table}[t]
    \centering
    \caption{Real-world graphs used in our experiments.}
    \label{tab:instance}
    \scalebox{0.75}{
    \begin{tabular}{lrrl}
    \toprule
    Name &$n$  &$m$ & Description  \\
    \midrule
 \texttt{Karate}             &34       &78     &Social network\\
 \texttt{Lesmis} & 77& 254  &Social network\\
 \texttt{Polbooks}           &105      &441    &Co-purchased network\\
 \texttt{Adjnoun}            &112      &425    &Word-adjacency network\\
 \texttt{Jazz}               &198      &2,742    &Social network\\
 \texttt{Email}              &1,133    &5,451   &Communication network\\
 \texttt{email-Eu-core }     & 986    & 16,064  & Communication network\\
  \texttt{Polblogs}           &1,222    &16,714   & Blog hyperlinks network \\
  \texttt{ego-Facebook}      &4,039    &88,234 &  Social network  \\
     \texttt{Wiki-Vote}           &7,066  &100,736  & Wikipedia ``who-votes-whom'' \\
    \bottomrule
    \end{tabular}
    }
\end{table}


\begin{algorithm}[t]
\caption{Baseline algorithm (\textsf{Naive})}
\label{alg:naive}
	\SetKwInOut{Input}{Input}
	\SetKwInOut{Output}{Output}
	\Input{ Number of iterations $T$ and a family of queryable subsets of at least $k$ vertices  $\mathcal{S} \subseteq  2^V $ }
	
	\Output{ $S \subseteq V$}
	
	$w_{\rm avg} \leftarrow \bm{0}$;
	
	$t_{\rm e} \leftarrow 0$ for $e \in E$;
	
	\For{$t=1,2,,\ldots,T$}{
	
	
	Choose $S_t \subseteq  \mathcal{S}$ uniformly at random;
	
	Call the sampling oracle for $S_t$ and observe $r_t(S_t)$;
	
	
	$t_e \leftarrow t_e+1$ for $e \in E(S_t)$;
	
	Update $w_{\rm avg}(e) \leftarrow \frac{w_{\rm avg}(e)(t_e-1)+ r_{S_t}/\ell}{t_e}$ for $e \in E(S_t)$;
	}
	
	$S \leftarrow $ Output of Charikar's LP-based exact algorithm~\cite{Charikar2000} for $G(V,E,w_{\rm avg})$;

    \Return{ $S$
    }
\end{algorithm}

\begin{table}[t]
    \centering
    \caption{Comparison between \textsf{DS-Lin} and the baseline algorithm (Algorithm~\ref{alg:naive}).
    }
    \label{tab:result}
    \scalebox{0.83}{
    \begin{tabular}{llrrrrrr}
    \toprule
    Graph & $k$  &　\textsf{DS-Lin} &\textsf{Naive} & \textsf{OPT} &$|S^*|$   \\
    \midrule  
            &$10$       &       111.08 & 19.94&   \\
 \texttt{Karate}     &$20$      &       111.08 &19.94 & 111.08 & 6  \\
            &$30$      &       111.08 & 19.94 &\\
\midrule  
            &$10$      &    179.72 &  177.19 && \\
 \texttt{Lesmis}     &$20$     &      179.72 & 177.19 &179.72 & 15  \\
            &$30$     &      179.72 & 177.19 &  \\
\midrule

            &$10$      &       227.43 & 172.69 && \\
 \texttt{Polbooks}   &$20$     &      227.62 & 172.69 & 228.67 & 19   \\
            &$30$     &       227.67 & 172.69  & \\

\midrule
            &$10$       &     133.23 & 53.27  &&   \\
 \texttt{Adjnoun}    &$40$      &     133.62 &  53.27 & 134.83 & 55   \\
            &$70$      &     133.53 & 53.27 & \\
             
\midrule

                    &$10$    & 598.39 &   170.03  & &   \\
 \texttt{Jazz}                &$40$   & 598.81  & 170.46  & 599.43  & 42 \\
                    &$70$ & 598.81  & 164.76 &  \\

\midrule
                    &$10$     & 223.36 & 67.24   &  &  \\
 \texttt{Email}             &$40$     & 223.37  & 67.24 &223.90 & 58 \\
                    &$70$     & 222.29& 67.24   \\
 \bottomrule
    \end{tabular}
    }
\end{table}

\begin{table*}[t]
\begin{center}
\caption{Performance of \textsf{DS-SR}. For \textsf{DS-SR} and \textsf{R-Oracle}, the quality of solutions, number of samples, and computation time are averaged over 100 executions.}\label{table:dssr}
\scalebox{0.78}{
\begin{tabular}{lrrrrrrrrrrrrrrr}
\toprule
Graph 
&\multicolumn{4}{c}{\textsf{DS-SR}}   
&\multicolumn{3}{c}{\textsf{R-Oracle}}
 &\textsf{G-Oracle} & \textsf{OPT}  \\
\cmidrule[0.4pt](r{0.1em}){2-5}
\cmidrule[0.4pt](l{0.25em}){6-8}
& $T$ & Quality   &\#Samples for single edges &Time(s)
&Quality   &\#Samples for single edges &Time(s)\\
\midrule
\texttt{Karate} &$10^3$ & 111.08  &  58 &0.00 &   111.08   &   10,296  &0.02 & 111.08 & 111.08 \\
\texttt{Lesmis} &$10^4$  & 177.66 & 752  &0.02  & 179.72 &         51,816  &0.07 & 176.29 &179.72   \\
\texttt{Polbooks} &$10^4$ &227.43&  419  &0.02 & 228.67& 214,767 & 0.22 & 227.47 & 228.67  \\
\texttt{Adjnoun} &$10^4$ &133.93&403&0.02&134.83& 241,400&0.26 &  133.97 & 134.83   \\
\texttt{Jazz} & $10^5$ &599.42&6,837&0.4 &599.43&1,115,994&1.49 & 599.43 & 599.43    \\
\texttt{Email} &$10^6$ &220.7&23,785&1.51&223.91&22,790,631&20.54 & 220.93 &223.90   \\
\texttt{email-Eu-core} &$10^6$  & 792.03  & 34,393 & 4.0  &792.19  &  17,509,760   &   29.69   & 792.07 &792.19 \\
\texttt{Polblogs} &$10^6$  &1211.37  &  16,508  &4.38  &  1211.44     & 18,452,256   &20.76 & 1211.44 &1211.44 \\  
\texttt{ego-Facebook} &$10^7$  &  2654.40&  103,546 &42.61 & 2783.85    &   78,175,324  &108.82 &2654.44  & 2783.85 \\ 
\texttt{Wiki-Vote} &$10^8$ &1235.71&3,975,994&425.42&1235.95&288,205,696&638.92 & 1235.76 &1235.95 \\ 
\bottomrule
\end{tabular}
}
\end{center}
\end{table*}


    

    
    
    
    


        

\subsection{Experiments for \textsf{DS-Lin}} 
 
\paragraph{Baseline.} We compare our algorithm with the following naive approach, which we refer to as \textsf{Naive}. 
As well as our proposed algorithm, \textsf{Naive} is a kind of algorithm that sequentially accesses a sampling oracle to estimate $w$ and uses uniform sampling strategy.
The entire procedure is detailed in Algorithm~\ref{alg:naive}.

\paragraph{Parameter settings.}
Here we use the graphs with up to ten thousand edges. 
We set the minimum size of queryable subsets $k=10$, 20, 30 for \texttt{Karate}, \texttt{Lesmis}, and \texttt{Polbooks}, and $k=10$, 40, 70 for \texttt{Adjnoun}, \texttt{Jazz}, and \texttt{Email}.
We construct $\mathcal{S}$ so that the matrix consisting of rows corresponding to the indicator vector of $S \in \mathcal{S}$ has rank $m$. Each $S \in \mathcal{S}$ is given as follows. We select an integer $\ell \in [k,n]$  and choose $S\subseteq V$ of size $\ell$ uniformly at random.
A uniform allocation strategy is employed by \textsf{DS-Lin} as $p$, i.e., $p=(1/|\mathcal{S}|)_{S \in \mathcal{S}}$.
We set $\lambda=100$ and $R=1$.
In our theoretical analysis, we provided an upper bound of the number of queries required by \textsf{DS-Lin} for $\epsilon>0$ and $\delta \in (0,1)$. 
However, such an upper bound is usually too large in practice.
Therefore, we terminate the while-loop of our algorithm once the number of iterations exceeds 10,000 except for the initialization steps. 
To be consistent, we also set $T=m+10000$ in \textsf{Naive}. 

\paragraph{Results.}
Here we compare our proposed algorithm \textsf{DS-Lin} with \textsf{Naive} in terms of the quality of solutions.  The results are summarized in Table~\ref{tab:result}. 
The quality of output $S$ is measured by its density in terms of $w$ which is unknown to the learner. 
For all instances, we run each algorithm for 10 times, and report the average value. 
The last two columns of Table~\ref{tab:result} represent the optimal value and the size of an optimal solution, respectively, 
As can be seen, our algorithm outperforms the baseline algorithm; 
in fact, our algorithm always obtains a nearly-optimal solution. 
It should be noted that this trend is valid even if $k$ is quite large; 
in particular, even if $k$ is larger than the size of the densest subgraph on the edge-weighted graph $G=(V,E,w)$, 
our algorithm succeeds in detecting a vertex subset that is almost densest in terms of $w$. 
We also report how the density of solutions approaches to such a quality and behavior of \textsf{DS-Lin}  with respect to the number of iterations in Appendix~\ref{apx:exp}. 

Finally, we briefly report the running time of our proposed algorithm with 10,000 iterations. 
For small-sized instances, \texttt{Karate}, \texttt{Lesmis}, \texttt{Polbooks}, and \texttt{Adjnoun}, the algorithm runs in a few minutes. 
For medium-sized instances, \texttt{Jazz} and \texttt{Email}, the algorithm runs in a few hours.

\subsection{Experiments for \textsf{DS-SR} }

\paragraph{Compared algorithms.}
To demonstrate the performance of \textsf{DS-SR} for Problem~\ref{prob:graph_fixbudget},
we also implement two algorithms \textsf{G-Oracle} and \textsf{R-Oracle}.
\textsf{G-Oracle} is the greedy peeling algorithm with the knowledge of the expected weight $w$~\cite{Charikar2000}, which is detailed in Algorithm~\ref{alg:greedy}. 
Note that we are interested in how the quality of solutions by \textsf{DS-SR} is close to that of \textsf{G-Oracle}. 
\textsf{R-Oracle} is the state-of-the-art robust optimization algorithm proposed by \citet{Miyauchi_Takeda_18} with the use of {\em edge-weight space} $W=\times_{e \in E}[\min \{ w(e)-1 ,0 \},w(e)+1]$, 
which is detailed in Algorithm~\ref{alg:roracle} in Appendix~\ref{apx:dssr}.
For \textsf{R-Oracle}, we set $\gamma=0.9$ and $\varepsilon=0.9$ as in \citet{Miyauchi_Takeda_18}.

\begin{algorithm}[t]
\caption{Greedy peeling (\textsf{G-Oracle})}
\label{alg:greedy}
	\SetKwInOut{Input}{Input}
	\SetKwInOut{Output}{Output}
	\Input{ Graph $G=(V,E,w)$}
	\Output{ $S \subseteq V$}
	$S_{|V|} \leftarrow V$;
	
    \For{$i \leftarrow |V|, \ldots, 2$}{
    
    Find $v_i \in \argmin_{v \in S_i} \text{deg}_{S_i}(v)$;
    
    $S_{i-1} \leftarrow S_i \setminus \{v_i\}$;
    
    }

    \Return{ $S_i \in \{ S_1, \ldots, S_{|V|} \} $ that maximizes $f_w(S)$
    }
\end{algorithm}

\paragraph{Results.} 
For \textsf{DS-SR}, in order to make $\tilde{T}_t$ positive, we run the experiments with a budget $T=10^{ \left \lceil \log_{10}\sum_{i=1}^{n+1}i  \right \rceil } $ for all instances.
The results are summarized in Table~\ref{table:dssr}.
The quality of output is again evaluated by its density in terms of $w$.
For \textsf{DS-SR} and \textsf{R-Oracle},
we list the total number of samples for individual edges used in the algorithms.
To observe the scalability, we also report the computation time of the algorithms.
We perform them 100 times on each graph. 
As can be seen, \textsf{DS-SR} required much less samples for single edges than that of \textsf{R-Oracle} but still can find high-quality solutions.
The quality of solutions by \textsf{DS-SR} is comparable with that of \textsf{G-Oracle}, which has a prior knowledge of expected weights $w$.
Moreover, in terms of computation time,
\textsf{DS-SR} efficiently works on large-sized graphs with about ten thousands of edges.
Finally, Figure~\ref{fig:sample} depicts the fraction of the size of edge subsets queried in \textsf{DS-SR} (see Appendix~\ref{apx:dssr} for results on all graphs).
We see that in the execution of \textsf{DS-SR}, the fraction of the number of queries for single edges is less than 30\%.

\begin{figure}[t!] 
\begin{center}
  \subfigure{
  \includegraphics[width=0.22\textwidth]{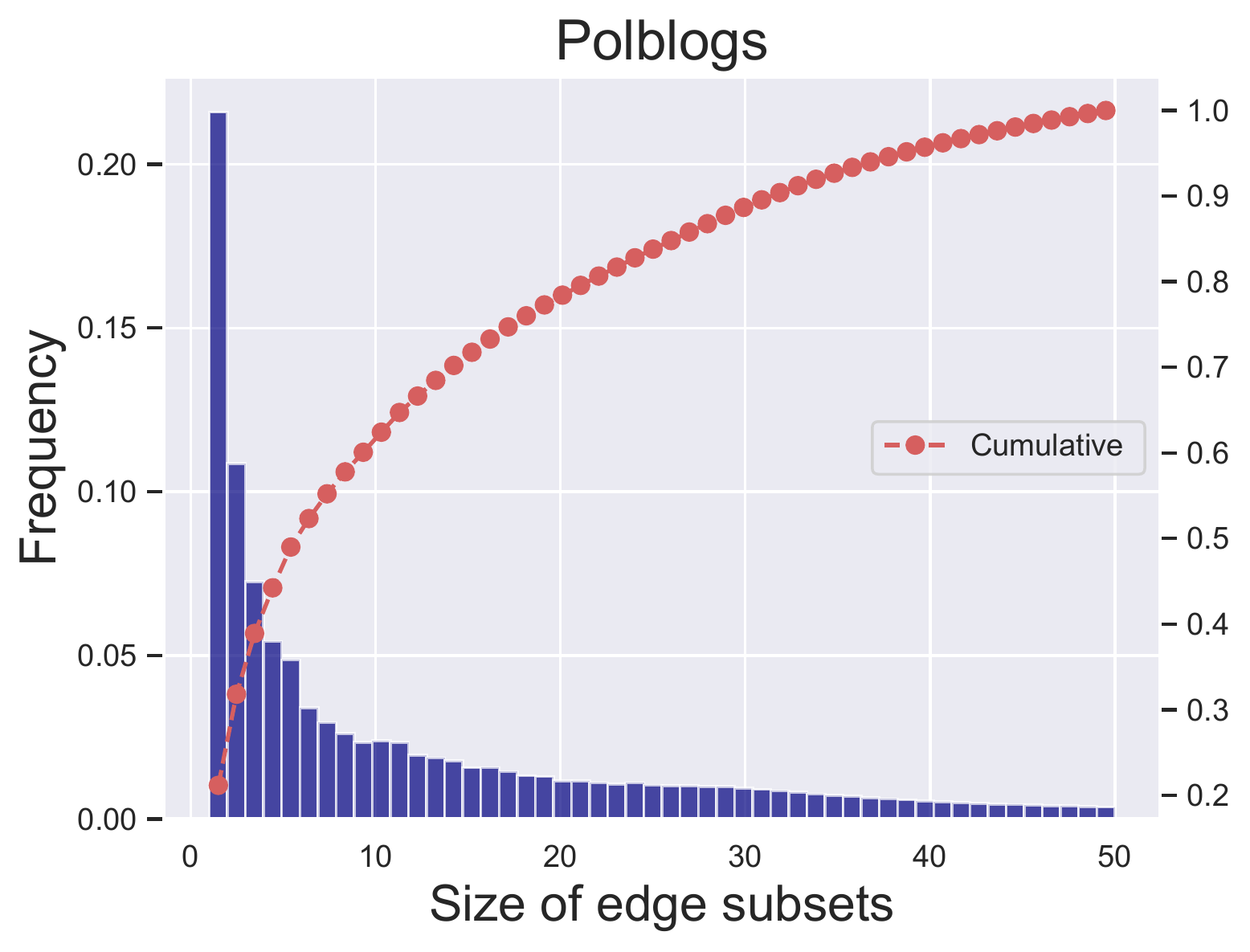}
  }
 \subfigure{
  \includegraphics[width=0.22\textwidth]{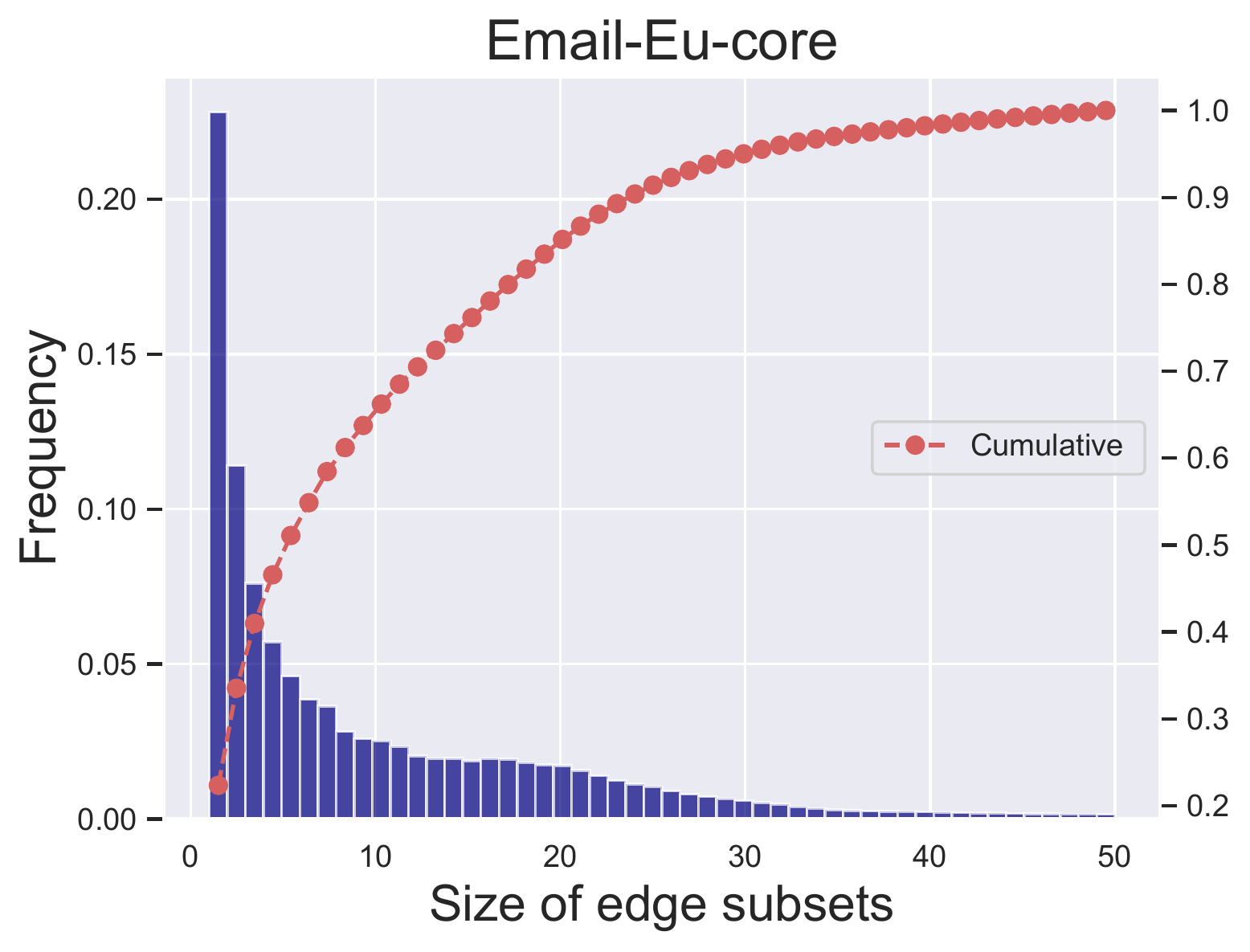}
  }\\
 \subfigure{
  \includegraphics[width=0.22\textwidth]{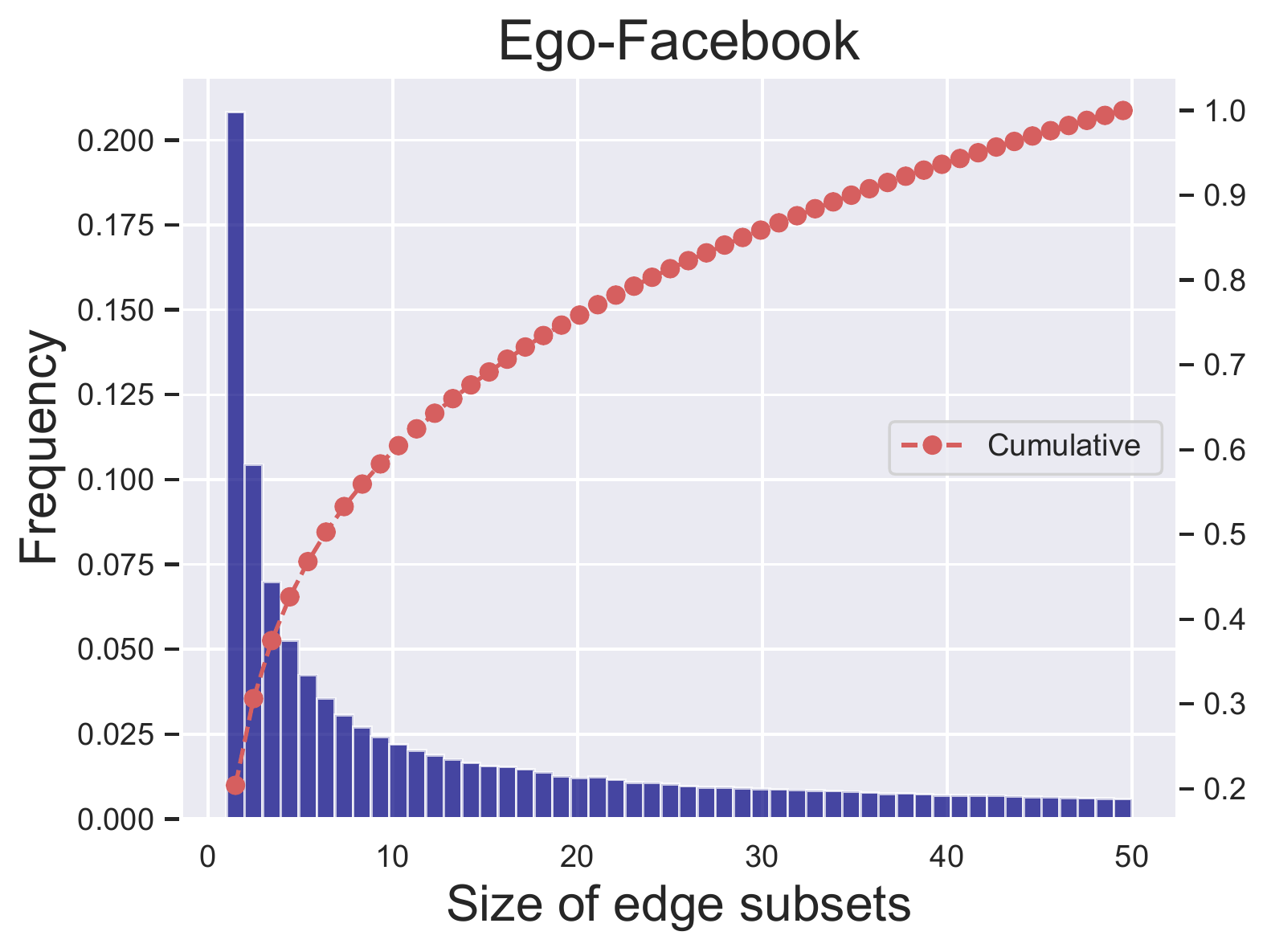}
  }
 \subfigure{
  \includegraphics[width=0.22\textwidth]{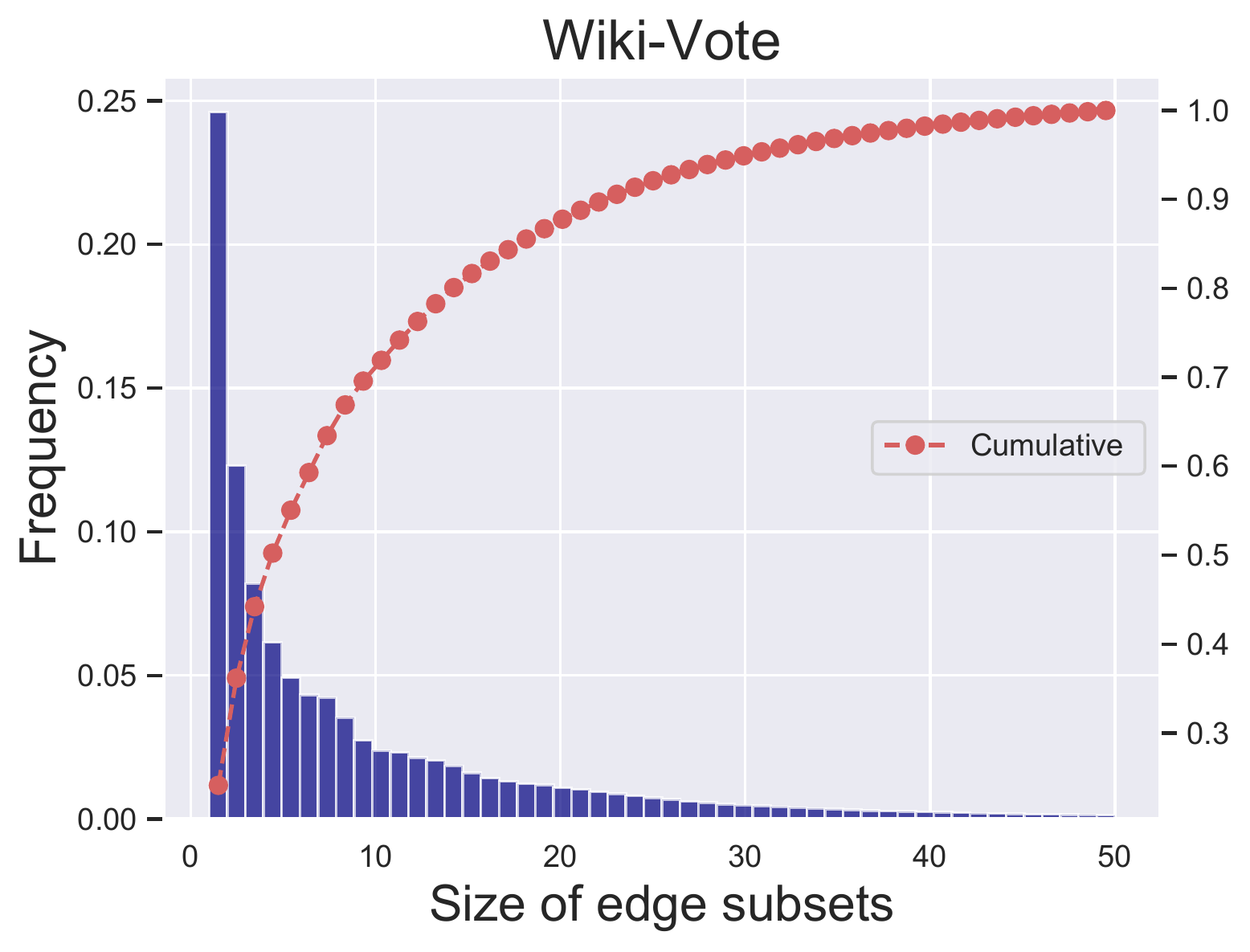}
  }
   \end{center}\caption{Fraction of the size of edge subsets queried in \textsf{DS-SR}. All values are averaged over 100 executions.}\label{fig:sample}
\end{figure}

\section{Conclusion}\label{sec:conclusion}
In this study,
we introduced a novel online variant of the densest subgraph problem by bringing the concepts of combinatorial pure exploration, which we refer to as the DS bandits.
We first proposed an $(\epsilon,\delta)$-PAC algorithm called  \textsf{DS-Lin}, and provided a polynomial sample complexity guarantee.
Our key technique is to utilize an approximation algorithm using SDP for confidence bound maximization.
Then, to deal with large-sized graphs, we proposed an algorithm called \textsf{DS-SR} by combining the successive reject strategy and the greedy peeling algorithm. We provided an upper bound of probability that the quality of the solution obtained by the algorithm is less than $\frac{1}{2}{\rm OPT}-\epsilon$.
Computational experiments using well-known real-world graphs demonstrate the effectiveness of our proposed algorithm.



\section*{Acknowledgments}
The authors thank the anonymous reviewers for their useful comments and
suggestions to improve the paper.
YK would like to thank Wei Chen and Tomomi Matsui for helpful discussion,
and also thank Yasuo Tabei, Takeshi Teshima, and Taira Tsuchiya for their feedback on the manuscript.
YK was supported by Microsoft Research Asia D-CORE program and KAKENHI 18J23034.
AM was supported by KAKENHI 19K20218. 
JH was supported by KAKENHI 18K17998.
MS was supported by KAKENHI 17H00757.

\bibliography{dense}
\bibliographystyle{icml2020}

\begin{onecolumn}
\icmltitle{Supplementary Material:\\
\vspace{2mm}
Online Dense Subgraph Discovery via Blurred-Graph Feedback}

\appendix

\section{Related work on the densest subgraph problem}\label{sec:related}


The densest subgraph problem has a large number of noteworthy problem variations. 
The most related one is the above-mentioned variant dealing with the uncertainty of edge weights, 
recently introduced by \citet{Miyauchi_Takeda_18}. 
Here we review their models in detail. 
They introduced two optimization models: 
the \emph{robust densest subgraph problem} and the \emph{robust densest subgraph problem with sampling oracle}. 
In both models, it is assumed that we have an edge-weight space $W=\times_{e\in E}[l_e,r_e]\subseteq \times_{e\in E}[0,\infty)$ that contains the unknown true edge weight $w$. 
That is, we have only the lower and upper bounds on the true edge weight for each edge. 
The key question they addressed is as follows: how can we evaluate the quality of solutions in this uncertain scenario?
To answer this question, they employed the measure called the \emph{robust ratio}, which is a well-known notion in the field of robust optimization. 
In the first model, given an unweighted graph $G$ and an edge-weight space $W$, we are asked to find $S\subseteq V$ that maximizes the robust ratio. 
In the second model, as mentioned above, we also have access to the edge-weight sampling oracle. 

There are other problem variations under uncertain scenarios. 
\citet{Zou_13} studied the densest subgraph problem on \emph{uncertain graphs}. 
Uncertain graphs are a generalization of graphs, which can model the uncertainty of the existence of edges (rather than the uncertainty of edge weights). 
More formally, an uncertain graph consists of an unweighted graph $G=(V,E)$ and a function $p:E\rightarrow [0,1]$, 
where $e\in E$ is present with probability $p(e)$ whereas $e\in E$ is absent with probability $1-p(e)$. 
In the problem introduced by \citet{Zou_13}, given an uncertain graph $G$, we are asked to find $S\subseteq V$ that maximizes the expected value of the density. 
\citet{Zou_13} demonstrated that this problem can be reduced to the original densest subgraph problem, 
and designed polynomial-time exact algorithm using the reduction. 
Very recently, \citet{Tsourakakis+19} introduced a novel optimization model, 
which they refer to as the \emph{risk-averse DSD}. 
In this model, given an uncertain graph, 
we are asked to find $S\subseteq V$ that has a large expected value of the density, 
at the same time, has a small risk. 
The risk of $S\subseteq V$ is measured by the probability that $S$ is not dense on a given uncertain graph. 
They showed that the risk-averse DSD can be reduced to \textsc{Neg}-DSD, 
and designed an efficient approximation algorithm based on the reduction. 

In addition to the above uncertain variants, the densest subgraph problem has many other interesting variations. 
In particular, the size-restricted variants have been actively studied~\cite{Andersen_Chellapilla_09,Bhaskara+10,Feige+01,Khuller_Saha_09}. 
For example, in the densest $k$-subgraph problem~\cite{Feige+01}, given an edge-weighted graph $G$ and a positive integer $k$, 
we are asked to find $S\subseteq V$ that maximizes the density $f_w(S)$ (or equivalently $w(S)$) subject to the size constraint $|S|=k$. 
It is known that such a size restriction makes the problem much harder; 
in fact, the densest $k$-subgraph problem is NP-hard and the best known approximation ratio is $\Omega(1/n^{{1/4}+\epsilon})$ for any $\epsilon>0$~\cite{Bhaskara+10}. 
The densest subgraph problem has also been extended to more general graph structures such as 
hypergraphs~\cite{Hu+17,Miyauchi+15} and multilayer networks~\cite{Galimberti+17}. 
Moreover, to cope with the dynamics of real-world graphs and to model the limited computation resources in reality, 
some literature has considered dynamic settings~\cite{Epasto+15,Hu+17,Nasir+17} and streaming settings~\cite{Angel+12,Bahmani+12,Bhattacharya+15,McGregor+15}, respectively. 
The average-degree density itself has also been generalized by 
modifying the numerator~\cite{Mitzenmacher+15,Miyauchi_Kakimura_18,Tsourakakis_15} or the denominator~\cite{Kawase_Miyauchi_18} of $d(S)=\frac{w(S)}{|S|}$, for some specific purposes.

\section{Notation}\label{apx:notation}
We give the summary of notation in Table~\ref{tab:notaiton}.

\begin{table}[t]
    \centering
    \caption{Notation.}
    \label{tab:notaiton}
    \scalebox{0.70}{
    \begin{tabular}{ll}
    \toprule
     Notation & Description  \\
    \midrule
 $G=(V,E,w)$              &Undirected graph\\
 $E(S)=\{\{u,v\} \in E \mid u,v \in S  \}$  &Subset of edges induced by $S \subseteq V$\\
  $G[S]=(S,E(S))$ &Subgraph induced by $S$\\
  $w: E \rightarrow \mathbb{R}_{>0}$ & Expected edge weights\\
 $w(S)=\sum_{e \in E(S)}w_e$ & Sum of edge weights in $E(S)$ \\
 $f_w(S)=w(S)/|S|$          & Degree density for weight $w$ and $S \subseteq V$\\
$\chiEofS \in \{0,1\}^E$     &Indicator vector of $E(S)$\\
${\bf x}_t=(\bm{\chi}_{\sm {$E(S_1)$}},\ldots,\bm{\chi}_{\sm {$E(S_t)$}}) \in \{0,1\}^{E \times t}$     &Sequence of indicator vectors\\
 $(r_1(S_1), \ldots, r_t(S_t)) \in \mathbb{R}^{ t}$  &Sequence of observed rewards\\
$A_{{\bf x}_t}=\sum_{i=1}^t \bm{\chi}_{\sm {$E(S_i)$}} \bm{\chi}_{\sm {$E(S_i)$}}^{\top} + \lambda I $    &  Design matrix\\
$b_{{\bf x}_t}$       & $b_{{\bf x}_t}=\sum_{i=1}^t \bm{\chi}_{\sm {$E(S_i)$}}r_i(S_i) \in \mathbb{R}^E$ \\
$\widehat{w}_t
=A_{{\bf x}_t }^{-1}b_{{\bf x}_t}\in \mathbb{R}^{E}$     & Regularized least-square estimator  \\
 $\|x\|_B=\sqrt{x^\top Bx}$ & Quadratic (ellipsoidal) norm\\
$ N(v) =\{ u \in V \mid \{ u, v\} \in E\}$&Set of neighbors of $v\in V$  \\ 
$ \text{deg}_{\max} =\max_{v \in V}|N(v)|$   & Maximum degree of vertices\\
$p=(p(S))_{S \in {\cal S}} \in {\cal P}$ & Predetermined proportions of queries\\
$\Lambda_p = \sum_{S \in {\cal S}} p(S) \chiEofS \chiEofS^{\top}$ &  Design matrix for $p \in {\cal P}$\\
 $\rho_{\Lambda_p}= \max_{x \in [-1,1]^m} \|x \|^2_{{\Lambda_p}^{-1}}$ & Upper bound of maximal confidence width\\
$N_S(v) =\{u \in S \mid \{u,v\} \in E \}$ & Set of neighboring vertices of $v$ in $G[S]$\\
$E_S(v) = \{\{u, v\}\in E \mid u\in N_S(v) \}$& Set of incident edges to $v$ in $G[S]$\\
 $\hat{X}_{F}(k)= \frac{1}{k} \sum_{s=1}^k X_{F}(s)$ & Empirical mean of weights for $k$ samples\\
  $\widehat{ \mathrm{deg}}_{S,v}(t)= \hat{X}_{E_{S}(v)} \left( T_{E(S)}(t) \right)$ & Empirical degree in $S \subset V$ for $v \in S$\\
 $ \widehat{f}(S_{n-t+1}) =\frac{\frac{1}{2}\sum_{v \in S_{n-t+1} \widehat{ \mathrm{deg}}_{S_{n-t+1}}(v,t)} }{|S_{n-t+1}|}$ & Empirical quality function at round $t$\\
    \bottomrule
    \end{tabular}
    }
\end{table}

\section{LP-based exact algorithm for the densest subgraph problem}\label{sec:exact}

We describe an exact algorithm for the densest subgraph problem based on the following LP and simple rounding procedure proposed by ~\citet{Charikar2000} which we use in our proposed algorithm.

\begin{alignat}{4}\label{prob:lp} 
& &\ \ &\text{maximize} & \  & \sum_{e \in E} w_e x_e   \notag \\
&              &    &\text{subject to} &    &  x_e \geq y_u,\  x_e \geq y_v \ \   & \forall e= \{ u,v\} \in E, \notag \\
&              &    &  &    & \sum_{v \in V}y_v =1, & \notag \\
&              &    & &    & x_e, y_v \geq 0 & \forall e \in E, \forall v \in V. 
\end{alignat}

Let $(x^*, y^*)$ be an optimal solution to the above LP.
For a real number $r \geq 0$,
the algorithm considers a sequence of subsets vertices $S(r)=\{ v \in V \mid y^*_v \geq r\}$
and finds $r^* \in \argmax_{r \in [0,1]}f_w(S(r))$.
Such a $r^*$ can be obtained by simply examining $r=y^*_v$ for each $v \in V$.
Finally, the algorithm outputs $S(r^*)$.
\citet{Charikar2000} proved that the output $S(r^*)$ is an optimal solution to the densest subgraph problem.

\section{Arm allocation strategy}\label{apx:allocation}

In this section, we introduce a possible allocation strategy that can be used in \textsf{DS-Lin} algorithm.
To reduce the number of samples,
good arm allocation strategy makes confidence bound shrinking fast.
We define the G-allocation for a family $\mathcal{S}$ as:
\[
p =\argmin_{p \in \mathcal{P}} \max_{S \subseteq \mathcal{S}}  \|\chiEofS \|^2_{{\Lambda_p}^{-1}}.
\]
 There are existing studies that proposed approximate solutions
 to solve it in the experimental design literature ~\citep{mustapha2010, SAGNOL2013};
we can solve a continuous relaxation of the problem by a projected gradient algorithm when the support size $|\mathcal{S}|$ is not so large.
For details on G-allocation or standard G-optimal design, see~\citet{Soare2014} or see~\citet{pukelsheim2006}.

In general, an algorithm that adaptively changes an arm selection strategy based on the past observations at each round, which is called an {\em adaptive algorithm}, 
is desired because samples should be allocated for comparison of near-optimal arms in order to reduce the number of samples.
On the other hand,
the algorithm that fixes all arm selections before observing any reward is called the {\em static algorithm}.
Although the static algorithm is not able to focus on estimating near-optimal arms,
it can be used to analyze the worst case optimality.
In our text, each arm corresponds to an edge set; it is rare that any set is able to query since the possible choices are exponential.
Therefore, we design a static algorithm \textsf{DS-Lin} for solving Problem~\ref{prob:graph}.
On the other hand,  if we are allowed to query any action as in Problem~\ref{prob:graph_fixbudget}, we can design an adaptive algorithm \textsf{DS-SR}.

\section{Technical lemmas for Theorem~\ref{thm:samplecomplexity}}\label{apx:thm1}




We introduce random event $\cE_t$ which characterizes
the event that the confidence bounds of
any feasible solution $ S \in V$
are valid at round $t$.
 We define a random event $\cE_t$ as follows:
 \begin{alignat}{4}\label{event}
&{\cal E}_t=\{ \forall S \in V \ {\rm and }  \  v \in S, 
  &\quad  \left| w(S)-\widehat{w}_t(S) \right| \leq C_t \| \chiEofS  \|_{\Atin} \}.
\end{alignat}

The following lemma states that event $\cE=\bigcap_{t=1}^{\infty} \cE_t$ occurs with high probability.
\begin{lemma}\label{lemma:highprob}
 The event ${\cal E} $ occurs with probability at least $1-\delta$.
\end{lemma}
The proof is omitted since it is straightforward from Proposition~1 and union bounds.

\begin{lemma}\label{lemma:epsilonoptimal}
For a fixed round $t > m$,
assume that $\cE_t$ occurs.
Then, if the algorithm stops at round $t$, the output of the algorithm $S_{\tt OUT}$ satisfies $f_w(S^*)-f_w(S_{\tt OUT})  \leq \epsilon$.
\end{lemma}

\begin{proof}


If $S_{\tt OUT}= S^*$, we obviously have the desired result.
Then, we shall assume $S_{\tt OUT} \neq S^*$.
\begin{alignat*}{5}
f_w(S_{\tt OUT}) & \geq f_{w_t}(S_{\tt OUT})-\frac{C_t \| \chiEofSout \|_{\Atin} }{| S_{\tt OUT}| } \\ \notag
&\geq 
\max_{ S \neq S_{\tt OUT} \mid S \subseteq  V}f_{\widehat{w}_t}(S)+ \frac{C_t Z_t}{2 \alpha } -\epsilon \\ \notag
&\geq 
f_{\widehat{w}_t}(S^*)+ \frac{ \max_{x \in [-1,1]^m}\| x \|_{\Atin}}{2 } -\epsilon \\ \notag
&\geq 
f_{\widehat{w}_t}(S^*)+ \frac{ \max_{S \subseteq V} \| \chiEofS \|_{\Atin}}{2 } -\epsilon \\ \notag
&\geq  f_{\widehat{w}_t}(S^*) +\frac{C_t \| \chiEofSstar\|_{\Atin}}{|S^*|}-\varepsilon  \hspace{0.8cm}
\\\notag
&\geq  f_w(S^*)-\varepsilon,
\end{alignat*}
where the first and last inequalities follow from the event $\cE_t$,
and the second inequality follows from the stopping condition,
and the third inequality follows from the definition of $Z_t$ and approximation ratio $\alpha$.
\end{proof}

In \citet{Miyauchi_Takeda_18},
they provided the following lemma, which we also use to prove Theorem~\ref{thm:samplecomplexity}.
\begin{lemma}[\citet{Miyauchi_Takeda_18}, Lemma~2]\label{lemma:miyauchi}
Let $G=(V,E)$ be an undirected graph.
Let $w_1$ and $w_2$ be edge-weight vectors such that $\| w_1-w_2\|_{\infty} \leq \beta$ holds for $\beta>0$.
Then, for any $S \subseteq V$, it holds that
\[
|f_{w_1}(S)-f_{w_2}(S) | \leq \sqrt{\frac{m}{2}} \cdot \beta.
\]
\end{lemma}

\section{Proof of Theorem~\ref{thm:samplecomplexity}}\label{apx:proof_theorem1}

\begin{proof}
We know that the event $\cE$ holds with probability at least $1-\delta$ from Lemma~\ref{lemma:highprob}.
Therefore, we only need to prove that, under event $\cE$,
the algorithm returns a set whose density is at least $f_w(S^*)-\epsilon$ and provide an upper bound of number of queries.
From Lemma~\ref{lemma:epsilonoptimal},
on the event $\cE$, the algorithm outputs $S_{\tt OUT} \subseteq V$ that satisfies $f_w(S^*)-f_w(S_{\tt OUT}) \leq \epsilon$.

Next, we focus on bounding the number of queries.
Recall that Algorithm~\ref{alg:main} employs a  stopping condition:
\begin{align}
&f_{\widehat{w}_t}(\widehat{S}_t)- \frac{C_t \|\chiEofShat \|}{|\widehat{S}_t|} \geq \max_{ S \neq \widehat{S}_t \mid S \subseteq  V}f_{\widehat{w}_t}(S)+ \frac{C_t Z_t}{2 \alpha } -\epsilon,
\end{align}
where $Z_t$ denotes the objective of the  approximate solution to P1.  
A sufficient condition of the stopping condition is that for $S^*$ and for $t >m$,
\begin{align}
    f_{\widehat{w}_t}(S^*)- \frac{C_t \| \chiEofSstar \|}{ | S^* |} \geq  \max_{S \neq S^* \mid S \subseteq V }f_{\widehat{w}_t}(S) + \frac{C_t Z_t}{2\alpha}-\epsilon,
\end{align}
Since $Z_t \leq \max_{x \in [-1,1]^m}\|x \|_{\Atin}  $ and $\| \chiEofSstar \| \leq \max_{x \in [-1,1]^m}\|x \|_{\Atin}$,
the following inequality is also a sufficient condition to stop:
\begin{align}\label{ineq:suf1}
        f_{\widehat{w}_t}(S^*)- \frac{C_t \max_{x \in [-1,1]^m}\|x \|_{\Atin} }{ | S^* |} \geq  \max_{S \neq S^* \mid S \subseteq V }f_{\widehat{w}_t}(S) + \frac{ C_t \max_{x \in [-1,1]^m}\|x \|_{\Atin} }{2\alpha}-\epsilon.
\end{align}
Recall that  $T_t(S)$ be the number of times that $S \subseteq {\cal S}$ is queried before $t$-th round in the algorithm.
We denote $p_{\bf{x}_t} $ by $ p_{\bf{x}_t} =(T_t(S)/t)_{S \in \mathcal{S}}$.
From the above definitions, the design matrix is $\At= t \Lambda_{p_{\bf{x}_t}}$. 
Recall that $\rho_{\Lambda_p}=\max_{x \in [-1,1]^m}\|x \|^2_{ {\Lambda^{-1}_p}}$,
and let $\rho_{\Lambda_{p_{\bf{x}_t}}}=\max_{x \in [-1,1]^m}\|x \|^2_{{\Lambda^{-1}_{p_{\bf{x}_t}}}}$.
From the fact that $ \lim_{t \to \infty}\Lambda_{p_{\bf{x}_t}} =\Lambda_p $,
for sufficiently large $t \gg m$  we have that $ |\rho_{\Lambda_p} -\rho_{\Lambda_{p_{\bf{x}_t}}}| \leq \epsilon$ with probability at least $1-\frac{\delta}{2}$ where $\delta \in (0,1)$ and $\epsilon>0$.
Let $ \bar{S}=\argmax_{S \neq S^* \mid S \subseteq V }f_{\widehat{w}_t}(S)$.
Then, \eqref{ineq:suf1} is rewritten as follows:
\begin{align}\label{ineq:suf2}
     f_{\widehat{w}_t}(S^*)-f_{\widehat{w}_t}(\bar{S}) \geq \left( \frac{1}{|S^*|}+ \frac{1}{2\alpha}  \right) C_t \sqrt{ \frac{\rho_{\Lambda_p} + \epsilon}{t} } -\epsilon.
\end{align}

Next, we show the following inequality.
\begin{align}\label{ineq:rhs_lower}
   f_{\widehat{w}_t}(S^*)-f_{\widehat{w}_t}(\bar{S}) \geq  \Delta_{\min}-\sqrt{2m} C_t \sqrt{ \frac{\rho_{\Lambda_p}+\epsilon}{t} }.
\end{align}
From Proposition~\ref{propo:yadokori}, with probability $1-\frac{\delta}{2}$,
we have  $\| w- \widehat{w}_t  \|_{\infty} \leq C_t \max_{S} \| \chiEofS \|_{\Atin}$.
From Lemma~\ref{lemma:miyauchi},
we see that $f_{\widehat{w}_t}(S^*)-f_w(S^*) \geq - \sqrt{\frac{m}{2}} C_t \max_{S \subseteq V} \| \chiEofS\|_{\Atin}$ and $f_w(\bar{S})-f_{\widehat{w}_t}(\bar{S}) \geq - \sqrt{\frac{m}{2}} C_t \max_{S \subseteq V} \| \chiEofS\|_{\Atin}$.
Therefore,
for sufficiently large $t \gg m$ such that $ |\rho_{\Lambda_p} -\rho_{\Lambda_{p_{\bf{x}_t}}}| \leq \epsilon$ with probability at least $1-\frac{\delta}{2}$, we have that
\begin{alignat*}{4} 
f_{\widehat{w}_t}(S^*)-f_{\widehat{w}_t}(\bar{S}) & \geq f_w(S^*)-f_w(\bar{S})-\sqrt{2m} C_t  \max_{S \subseteq V} \| \chiEofS \|_{\Atin} \\
& \geq f_w(S^*)-f_w(\bar{S})-\sqrt{2m} C_t  \max_{x \in \{0,1\}^m } \| x \|_{\Atin} \\
& = f_w(S^*)-f_w(\bar{S})-\sqrt{2m} C_t \sqrt{ \frac{\rho_{\Lambda_{p_{\bf{x}_t}}}}{t} }\\ 
& \geq f_w(S^*)-f_w(\bar{S})-\sqrt{2m} C_t \sqrt{ \frac{\rho_{\Lambda_p}+\epsilon}{t} }\\
& \geq  \Delta_{\min}-\sqrt{2m} C_t \sqrt{ \frac{\rho_{\Lambda_p}+\epsilon}{t} }.
\end{alignat*}
Then, we obtain \eqref{ineq:rhs_lower}.
Combining \eqref{ineq:suf2} and \eqref{ineq:rhs_lower},
we see that the following inequality is a sufficient condition to stop:
\begin{align}\label{ineq:suf3}
    \Delta_{\min}-\sqrt{2m} C_t \sqrt{ \frac{\rho_{\Lambda_p}+\epsilon}{t} } \geq \left( \frac{1}{|S^*|}+ \frac{1}{2\alpha}  \right) C_t \sqrt{ \frac{\rho_{\Lambda_p}+\epsilon}{t} } -\epsilon.
\end{align}
Solving \eqref{ineq:suf3} with respect to $t$,
we obtain 
\begin{align}
    t \geq (\sqrt{2m} + |S^*|^{-1} + 2\alpha^{-1})^2 C_t^ 2 \He,
\end{align}
where recall that $\He$ is defined as 
\begin{align*}
    \He= \frac{ \rho_{\Lambda_p}+\epsilon }{ (\Delta_{\min} + \epsilon)^2}.
\end{align*}

Therefore, from the above, we obtain $ t \geq  \left(\sqrt{2m} + |S^*|^{-1} + {2\alpha}^{-1} \right)^2 C_t^ 2 \He$ as a sufficient condition to stop.
Let $\tau >m$ be the stopping time of the algorithm.
From the above discussion and $|S^*| \geq 2$,
we see that 
\begin{align}\label{ineq:uppre_tau}
    \tau \leq  \left(\sqrt{2m} + \frac{\alpha+1}{2}\right)^2 C_\tau^ 2 \He.
\end{align}
Now we bound $C_{\tau}$ in~\eqref{ineq:uppre_tau}.
We have $\det(\Atau) \leq (\lambda + \tau )^m$,
 which is obtained by Lemma~10 in \citet{yadkori11}.
Then,
we obtain the following inequality in the similar manner in Theorem 2 in \citet{Xu2018}:
\begin{alignat}{4}\label{ineq:ctbound}
C_{\tau}& \leq  R' \sqrt{2 \log \frac{\det(\Atau)^{\frac{1}{2}} \det(\lambda  I)^{-\frac{1}{2}}}{\delta} } +\lambda^{\frac{1}{2}}L \notag \\
& \leq  R' \sqrt{2 \log \frac{1}{\delta}       
+ m \log \left( 1+ \frac{\tau}{\lambda}  \right)
}+\lambda^{\frac{1}{2}}L.
\end{alignat}

Using~\eqref{ineq:ctbound}, we give an upper bound of $\tau$.
We also use a similar proof strategy as in \citet{Xu2018}.
First, let us consider the case $\lambda>4 m(\sqrt{m}+\sqrt{2})^2  R'^2 \He$, where recall that $R'=\sqrt{ \mathrm{deg}_{\max}}R$.
Using the facts that $(a+b)^2 \leq 2(a^2+b^2)$ for $a,b>0$ and $\log (1+x) \leq x$,
we have 
\begin{alignat*}{4}
\tau & \leq \left(\sqrt{2m} + \frac{\alpha+1}{2}\right)^2 C_{\tau}^2 \He 
 \leq 2\left(\sqrt{2m} + \frac{\alpha+1}{2}\right)^2\He \left(  4 R'^2  \log \frac{1}{\delta} +\frac{R'^2 m }{\lambda}\tau  +\lambda L^2 \right).
\end{alignat*}
Thus, we obtain
\begin{alignat*}{4}
\tau & \leq \left( 1-2\left(\sqrt{2m} + \frac{\alpha+1}{2}\right)^2 \frac{R'^2 m \He}{\lambda} \right)
^{-1} ( 8 \He R'^2 \log 1/ \delta +2\He \lambda L^2 ) \\
& \leq 2 (8 \He R'^2 \log 1/ \delta +2\He \lambda L^2 ),
\end{alignat*}
where the last inequality holds from $\lambda>4 m(\sqrt{m}+\sqrt{2})^2  R'^2 \He$ and $\alpha=\sqrt{\frac{4}{7}}$.
Therefore, in this case,
we obtain 
\[
\tau=O \left(\left(\mathrm{deg}^2_{\max} R^2 \log \frac{1}{\delta}+ \lambda L^2 \right)  \He \right).
\]

Second, we consider $\lambda \leq  \frac{R'^2}{L^2} \log \left(\frac{1}{\delta}\right)$.
From this bound and~\eqref{ineq:ctbound},
we have
\begin{alignat*}{4}
C_{\tau}& 
\leq  2R' \sqrt{ 2 \log \frac{1}{\delta} + m \log \left( 1+ \frac{ \tau}{\lambda} \right) }.
\end{alignat*}
Let $V(m, \alpha)=\left(\sqrt{2m} + \frac{\alpha+1}{2}\right)^2$.
Therefore,
we obtain
\begin{alignat*}{4}
\tau & \leq V(m, \alpha) C_{\tau}^2\He 
\leq 4 V(m, \alpha) R'^2 \left( 2 \log \frac{1} {\delta}   + m \log \left( 1+\frac{ \tau}{\lambda}  \right)\right) \He.
\end{alignat*}
Let $N=8  V(m, \alpha)R'^2 \log  \frac{1} {\delta}  \He$ and let $\tau'$ be a parameter satisfying:
\begin{align}\label{parameta_tauprime}
    \tau =  N
+ 4V(m, \alpha)R'^2 m \log \left( 1+\frac{ \tau'}{\lambda}  \right)\He.
\end{align}
It is easy to see $\tau'  \leq \tau$.
Then, we have 
\begin{alignat*}{3}
\tau'  &\leq \tau =N+ 4V(m, \alpha)  R'^2 m \log \left( 1+\frac{ \tau'}{\lambda}  \right) \He\\
& \leq N+ 4 V(m, \alpha)R'^2 m  \sqrt{\frac{\tau'}{\lambda}} \He,
\end{alignat*}
where the second inequality follows from the fact $\log (1+x) \leq \sqrt{x}$.
Solving this quadratic inequality for $\sqrt{\tau'}$,
 we have
 \begin{alignat}{4}\label{tauprimebound}
 \sqrt{\tau'}& \leq \frac{2 V(m, \alpha)R'^2 m \He}{\sqrt{\lambda}}+
 \sqrt{ \frac{4V(m, \alpha)^2  R'^4 m^2 \He^2 }{\lambda} + N} \notag\\
 &\leq 2 \sqrt{ \frac{4 V(m, \alpha)^2 R'^4m^2 \He^2}{\lambda} +N}.
\end{alignat}
Let $M=2 \sqrt{ \frac{4 V(m, \alpha)^2R'^4m^2 \He^2 }{\lambda} +N}$.
We can give an upper bound $\tau$ by~\eqref{parameta_tauprime} and~\eqref{tauprimebound} as follows:
\begin{align*}
    \tau \leq 8 V(m, \alpha)R'^2 \He \log \frac{1}{\delta} + 4 \He V(m, \alpha)R'^2 m \log \left( 1+\frac{M}{\lambda} \right).
\end{align*}

Note that $\log M= O\left( \log \left( R'^4m^4 \He^2 + m R'^2 \log \frac{1}{\delta}\He   \right) \right)$ since $ V(m, \alpha)=O(m)$.
From the above and $R'=\sqrt{\mathrm{deg}_{\max}}R$,
we obtain
\begin{alignat*}{4}
\tau =  & O \left( m \mathrm{deg}_{\max}R^2 \He \log \frac{1}{\delta} +m\mathrm{deg}_{\max} R^2 \He  \log \left(\mathrm{deg}_{\max}^2 R^4 m^4 \He^2 + m \mathrm{deg}_{\max}R^2 \He  \log \frac{1}{\delta} \right) \right)\\
& = O \left( m \mathrm{deg}_{\max}R^2 \He \log \frac{1}{\delta}+C(\He, \delta) \right) 
\end{alignat*}
where 
\begin{alignat*}{4}
    & C(\He, \delta)=m\mathrm{deg}_{\max} R^2 \He  \log \left(\mathrm{deg}_{\max}^2 R^4 m^4 \He^2 + m \mathrm{deg}_{\max}R^2 \He  \log \frac{1}{\delta} \right) \\
    & = O\left( m\mathrm{deg}_{\max} R^2 \He  \log \left(\mathrm{deg}_{\max}R m \He \log \frac{1}{\delta} \right) \right).
\end{alignat*}
\end{proof}

\section{Technical lemmas for Theorem~\ref{thm:fixbudget}}

First we introduce a standard concentration inequality of sub-Gaussian random variables~\cite{Chen2014}.
\begin{lemma}[\citet{Chen2014}, Lemma~6]\label{lemma:hoef}
Let $X_1, \ldots, X_k$ be $k$ independent random variables such that, for each $i \in [k]$, random variable $X_i- \mathbb{E}[X_i]$ is R-sub-Gaussian distributed, i.e., $\forall a \in \mathbb{R}$, $\mathbb{E}[\exp(aX_i-a \mathbb{E}[X_i])] \leq \exp(R^2a^2)/2$. 
Let $\bar{X}= \frac{1}{k} \sum_{i=1}^k X_i$ denote the average of these random variables.
Then, for any $\lambda >0$, we have
\[
\Pr \left[ |\bar{X}- \mathbb{E}[\bar{X}] | \geq \lambda \right] \leq 2 \exp \left( - \frac{k\lambda^2}{2R^2} \right).
\]

\end{lemma}

For $S\subseteq V$, $v\in S$, and expected weight $w$, 
we denote by $\mathrm{deg}^*_S(v)$ the weighted degree of $v\in S$ on $G[S]$ in terms of the true edge weight $w$.
We show the following lemma used for analysis of Algorithm~\ref{alg:peeling_approxx2}.
\begin{lemma}\label{lemma:fix_budget_lemma}
Given an phase $t \in \{1,\ldots,n-1\}$, 
we define random event 
\begin{align}
    \mathcal{E}_t' = \left\{ \forall v \in S_{n-t+1}, \left| \mathrm{deg}^*_{S_{n-t+1}}(v) - \widehat{\mathrm{deg}}_{S_{n-t+1}}(v,t) \right| \leq \epsilon \right\}.
\end{align}
Then, we have 
\begin{align}
 \Pr   \left[  \bigcap_{t=1}^{n-1} \mathcal{E}_t' \right] \geq 1- \frac{2 \mathrm{deg}_{\max}2^n R^2}{\epsilon^2}(n+1)^3 \exp \left( - \frac{(T-\sum_{i=1}^{n+1}i) \epsilon^2 }{4 n^2 \mathrm{deg}_{\max} \tilde{ \log }(n-1) } \right).
\end{align}

\end{lemma}
\begin{proof}

For any $S \subseteq V$, and $v \in S$,
recall that $X_{E_S(v)}(i)$ is $i$-th observation of edge-weights $E_S(v)$
for $i \in [k]$.
Then,
$X_{E_S(v)}(i)- \mathbb{E}[X_{E_S(v)}(i)]$ follows 
a $(\sqrt{|E_S(v)|}$ $R)$-sub-Gaussian
distribution.
We can assume that the sequence of weights for each subset of edges is drawn before the beginning of the game.
Thus $\hat{X}_{E_S(v)}(k)$ is well defined even if $E_S(v)$ has not been actually sampled $k$ times.
Therefore, from Lemma~\ref{lemma:hoef},  for any $\epsilon >0$
we have that 
\begin{align}\label{ineq:hoeff_graph}
\Pr \left[  \left| \mathbb{E}[{\hat{X}_{E_S(v)}(k)}]- \hat{X}_{E_S(v)}(k)\right|  \geq \epsilon \right] \leq 2 \exp \left( - \frac{k\epsilon ^2}{2 |E_S(v)|R^2} \right) \leq  2 \exp \left( - \frac{k\epsilon ^2}{2  \mathrm{deg}_{\max}R^2} \right) .    
\end{align}


  
  Fix $t \in \{ 1,\ldots, n-1 \}$ and fix a vertex $v \in S_{n-t+1}$ in a phase $t$.
If $|N_{S_{n-t+1}}(v)| =0$, it is obvious that $ \mathrm{deg}_{S_{n-t+1}}(v) =\widehat{\mathrm{deg}}_{S_{n-t+1}}(v) =0$.
Therefore, we will consider a vertex $v \in S_{n-t+1}$ such that $|N_{S_{n-t+1}}(v)| \geq 1$ in the rest of the proof.
  By the definition of $\tilde{T}_t$ for $t \in [ 1,2, \ldots, n-1]$,
  we have \[
  \tilde{T}_t \geq \frac{T-\sum_{i=1}^{n+1}i}{(n-t) \tilde{\log}(n-1)} \geq \frac{T-\sum_{i=1}^{n+1}i}{n \tilde{\log}(n-1)}.
  \]
 Then for $T_{E_{S_{n-t+1}}(v)}(t)$, we have:
\begin{align}\label{ineq:times}
T_{E_{S_{n-t+1}}(v)}(t)  = \sum_{i=1}^t \tau_i
=\sum_{i=1}^t (T_i'- T'_{i-1}) 
= T_t'  
\geq \frac{\tilde{T}_t}{2|S_{n-t+1}|}
\geq  \frac{T-\sum_{i=1}^{n+1}i}{ 2n^2\tilde{\log}(n-1)}.
\end{align}

Let $k'$ be the RHS of ~\eqref{ineq:times}, i.e. $k' =\frac{T-\sum_{i=1}^{n+1}i}{ 2n^2\tilde{\log}(n-1)}$.
For $v \in S_{n-t+1}$, we have
\begin{alignat}{4}\label{ineq:deg}
 &\Pr \left[  \left| \mathrm{deg}^*_{S_{n-t+1}}(v) - \widehat{\mathrm{deg}}_{S_{n-t+1}}(v,t) \right| \geq \epsilon \right] \notag \\
  & =\Pr \left[  \left| \mathbb{E}[{\hat{X}_{E_{S_{n-t+1}}(v)}(T_{E_{S_{n-t+1}}(v)}(t))}]- \hat{X}_{E_{S_{n-t+1}}(v)}(T_{E_{S_{n-t+1}}(v)}(t))\right|  \geq \epsilon \right] \notag \\
  &\leq \Pr \left[  \exists S \subseteq V, u \in S \middle| \left| \mathbb{E}[{\hat{X}_{E_{S}(u)}(T_{E_{S_{n-t+1}}(v)}(t))}]- \hat{X}_{E_{S}(u)}(T_{E_{S_{n-t+1}}(v)}(t))\right| \geq \epsilon  \right]  \notag \\
&\leq \sum_{S \in 2^V}  \sum_{ u \in S}　\sum_{k=k'}^{\infty} \Pr \left[ \left| \mathbb{E}[{\hat{X}_{E_{S}(u)}(k)}]- \hat{X}_{E_{S}(u)}(k)\right| \geq \epsilon  \right]  \notag \\
& \leq \sum_{S \in 2^V}  \sum_{ u \in S}　\sum_{k=k'}^{\infty} 2 \exp \left( - \frac{k\epsilon ^2}{2  \mathrm{deg}_{\max}R^2} \right) \notag \\
& = \sum_{S \in 2^V}  \sum_{ u \in S}　\frac{ 2 \exp \left( - \frac{k'\epsilon ^2}{2  \mathrm{deg}_{\max}R^2} \right) }{\exp \left({\frac{\epsilon^2}{2\mathrm{deg}_{\max}R^2}} \right)-1}\notag \\
& \leq  \sum_{S \in 2^V}  \sum_{ u \in S}　\frac{ 2 \exp \left( - \frac{k' \epsilon ^2}{2  \mathrm{deg}_{\max}R^2} \right) }{ {\frac{\epsilon^2}{2\mathrm{deg}_{\max}R^2}}}\notag \\
 & = \sum_{S \in 2^V}  \sum_{u \in S}  \frac{4 \mathrm{deg}_{\max}R^2}{\epsilon^2} \exp \left( - \frac{(T-\sum_{i=1}^{n+1}i) \epsilon^2 }{4 n^2 \mathrm{deg}_{\max}R^2 \tilde{ \log }(n-1)  } \right) \notag \\
&  \leq  \frac{4 \mathrm{deg}_{\max}2^n n  R^2}{\epsilon^2} \exp \left( - \frac{(T-\sum_{i=1}^{n+1}i) \epsilon^2 }{4 n^2 \mathrm{deg}_{\max}R^2  \tilde{ \log }(n-1) } \right),
\end{alignat}
where
the third inequality follows by \eqref{ineq:hoeff_graph} and
the fourth inequality follows by
$\mathrm{e}^{-x} \geq 1-x$.


Now using ~\eqref{ineq:deg} and taking a union bound for all $t \in \{1, \ldots n-1 \}$ and all $v \in S_{n-t+1}$,
we obtain
\begin{alignat*}{4}
    \Pr   \left[  \bigcap_{t=1}^{n-1} \mathcal{E}_t' \right]= & 1- \Pr \left[ \exists t \in \{1, \ldots n-1 \}, v \in S_{n-t+1} \middle| \left| \mathrm{deg}^*_{S_{n-t+1}}(v) - \widehat{\mathrm{deg}}_{S_{n-t+1}}(v,t) \right| \geq \epsilon \right] \\
     & \geq  1- \sum_{t=1}^{n-1}\sum_{v \in S_{n-t+1}} \Pr \left[  \left| \mathrm{deg}^*_{S_{n-t+1}}(v) - \widehat{\mathrm{deg}}_{S_{n-t+1}}(v,t) \right| \geq \epsilon \right] \\
     & \geq 1  -   \sum_{t=1}^{n-1}\sum_{v \in S_{n-t+1}} \frac{4 \mathrm{deg}_{\max}2^n n R^2 }{\epsilon^2}\exp \left( - \frac{(T-\sum_{i=1}^{n+1}i) \epsilon^2 }{4 n^2 \mathrm{deg}_{\max}R^2  \tilde{ \log }(n-1) } \right) \\
     & = 1  -  \frac{4 \mathrm{deg}_{\max}2^n n R^2 }{\epsilon^2} \sum_{t=1}^{n-1}|S_{n-t+1}| \exp \left(- \frac{(T-\sum_{i=1}^{n+1}i) \epsilon^2 }{4 n^2 \mathrm{deg}_{\max}R^2  \tilde{ \log }(n-1) } \right) \\
     & = 1  -  \frac{4 \mathrm{deg}_{\max}2^n n R^2 }{\epsilon^2} \sum_{t=1}^{n-1}(n-t+1) \exp \left(- \frac{(T-\sum_{i=1}^{n+1}i) \epsilon^2 }{4 n^2 \mathrm{deg}_{\max}R^2  \tilde{ \log }(n-1) } \right) \\
     & = 1- \frac{2 \mathrm{deg}_{\max}2^n R^2}{\epsilon^2} n^2(n+1) \exp \left(- \frac{(T-\sum_{i=1}^{n+1}i) \epsilon^2 }{4 n^2 \mathrm{deg}_{\max}R^2   \tilde{ \log }(n-1) } \right) \\
    & \geq 1- \frac{2 \mathrm{deg}_{\max}2^n R^2}{\epsilon^2}(n+1)^3 \exp \left( - \frac{(T-\sum_{i=1}^{n+1}i) \epsilon^2 }{4 n^2 \mathrm{deg}_{\max}R^2  \tilde{ \log }(n-1) } \right).
\end{alignat*}

\end{proof}

\section{Proof of Theorem~\ref{thm:fixbudget}}\label{apx:proof_theorem2}

\begin{proof}
First, we verify that the algorithms requires at most $T$ queries.
In each phase $t$, the number of samples Algorithm~\ref{alg:sampling} requires is at most $\tilde{T_t}+ |S_{n-t+1}|$, since we have that 
\begin{align*}
\sum_{v \in S_{n-t+1}} \sum_{i=1}^t \tau_i 
=\sum_{v \in S_{n-t+1}} \sum_{i=1}^t T'_i - T'_{i-1} 
=\sum_{v \in S_{n-t+1}} T'_t 
\leq \sum_{v \in S_{n-t+1}}\left( \frac{\tilde{T}_t}{|S_{n-t+1}|}+1 \right)
\leq \tilde{T}_t+ |S_{n-t+1}|.
\end{align*}
Therefore,
the total number of queries used by the algorithm is bounded by
\begin{alignat*}{4}
\sum_{t=1}^{n-1}\left( \tilde{T}_t+ |S_{n-t+1}| \right)
& \leq  \sum_{t=1}^{n-1}\tilde{T_t}+\sum_{t=1}^{n-1}(n-t+1)   \\
& \leq \sum_{t=1}^{n-1} \left(  \frac{T-\sum_{i=1}^{n+1}i}{(n-t) \tilde{\log}(n-1) } +1 \right)+  \sum_{t=1}^{n-1}(n-t+1) \\
& \leq \sum_{t=1}^{n-1}  \frac{T-\sum_{i=1}^{n+1}i}{(n-t) \tilde{\log}(n-1) }+  \sum_{i=1}^{n+1}i \\
& = \frac{(T-\sum_{i=1}^{n+1}i)}{\tilde{\log}(n-1)} \tilde{\log}(n-1)  +\sum_{i=1}^{n+1}i \\
& = T- \sum_{i=1}^{n+1}i+\sum_{i=1}^{n+1}i=T.
\end{alignat*}

Lemma~\ref{lemma:fix_budget_lemma} implies that
the random event $\cE' := \bigcap_{t=1}^{n-1} \cE'_t$ occurs with probability at least $  1- \frac{2 \mathrm{deg}_{\max}2^n R^2}{\epsilon^2}(n+1)^3 \exp \left( - \frac{(T-\sum_{i=1}^{n+1}i) \epsilon^2 }{4 n^2 \mathrm{deg}_{\max}R^2  \tilde{ \log }(n-1) } \right)$.
We shall assume the event $\cE'$ occurs in the rest of the proof,
because we only need to show that the algorithm outputs a solution $S_{\tt OUT}$ that guarantees $f_w(S_{\tt OUT}) \geq \frac{f_w(S^*)}{2} - \epsilon$ under $\cE'$.

Let $S^*\subseteq V$ be an optimal solution in terms of the expected weight $w$.
Choose an arbitrary vertex $v\in S^*$. 
From the optimality of $S^*\subseteq V$, it holds that
\begin{align*}
f_w(S^*)=\frac{w(S^*)}{|S^*|}\geq \frac{w(S^*\setminus \{v\})}{|S^*|-1}=f_w(S^*\setminus \{v\}). 
\end{align*}
By using the fact that $w(S^*\setminus \{v\}) = w(S^*) - \mathrm{deg}^*_{S^*}(v)$, the above inequality can be transformed into 
\begin{align}\label{ineq:optimality}
\mathrm{deg}^*_{S^*}(v)\geq f_w(S^*). 
\end{align}

Let $S_{\tau} \subseteq V$ be the last subset over the phases that satisfies $S_{\tau} \supseteq S^*$ and let $\tau \in [1, \ldots, n]$ be its phase.
Let $\tau_{\tt OUT}$ be the phase $t$ such that $S_{n-t+1}=S_{\tt OUT}$.
Then we have 
\begin{align*}
f_w(S_{\tt OUT})
&=\frac{\frac{1}{2}\sum_{v\in S_{\tt OUT}}\mathrm{deg}^*_{S_{\tt OUT}}(v)}{|S_{\tt OUT}|}\notag\\
&\geq \frac{\frac{1}{2}\sum_{v\in S_{\tt OUT}} \left( \widehat{ \mathrm{deg}}_{S_{\tt OUT}}(v,\tau_{\tt OUT}) -\epsilon \right)}{|S_{\tt OUT}|} \notag\\
&\geq \frac{\frac{1}{2}\sum_{v\in S_\tau}
\widehat{ \mathrm{deg}}_{S_{\tau}}(v, \tau)}{|S_{\tau}|}-\frac{\epsilon}{2}.
  \end{align*}
where
the first inequality follows from event $\cE'$,
the second inequality follows from the greedy choice of $S_{\tt OUT}$.
Recall that  the algorithm removes the vertex that satisfies $v_\tau \in \argmin_{v \in S_{\tau}} \widehat{ \mathrm{deg}}_{S_{\tau}}(v, \tau)$ in the phase $\tau$.
Therefore, from the definition of $S_\tau$, it is clear that $v_\tau \in S^*$.
Using this property, we further have that
\begin{align*}
\frac{\frac{1}{2}\sum_{v\in S_{\tau}}
\widehat{ \mathrm{deg}}_{S_{\tau}}(v, \tau)}{|S_{\tau}|}-\frac{\epsilon}{2}
& \geq \frac{\frac{1}{2}\sum_{v \in S_{\tau}}
\widehat{ \mathrm{deg}}_{S_{\tau}}(v_\tau, \tau)}{|S_{\tau}|}-\frac{\epsilon}{2} \\
&=\frac{1}{2}
\widehat{ \mathrm{deg}}_{S_{\tau}}(v_\tau, \tau)-\frac{\epsilon}{2} \\
& \geq \frac{1}{2}
 \mathrm{deg}^*_{S_{\tau}}(v_\tau)-\epsilon \\
 & \geq \frac{1}{2}
 \mathrm{deg}^*_{S^*}(v_\tau)-\epsilon \\
  & \geq \frac{1}{2}
 f_w(S^*)-\epsilon,
\end{align*}
where the second inequality follows from event $\cE'$, and third inequality follows  from the fact $S_{\tau}\supseteq S^*$, 
and the last inequality follows from the fact $v_{\tau} \in S^*$ and inequality (\ref{ineq:optimality}). 
Therefore, we obtain $f_w(S_{\tt OUT}) \geq \frac{1}{2}
 f_w(S^*)-\epsilon$. That concludes the proof.
\end{proof}


\begin{figure*}[t!]
\begin{center}
 \subfigure{
   \includegraphics[width=0.27\textwidth]{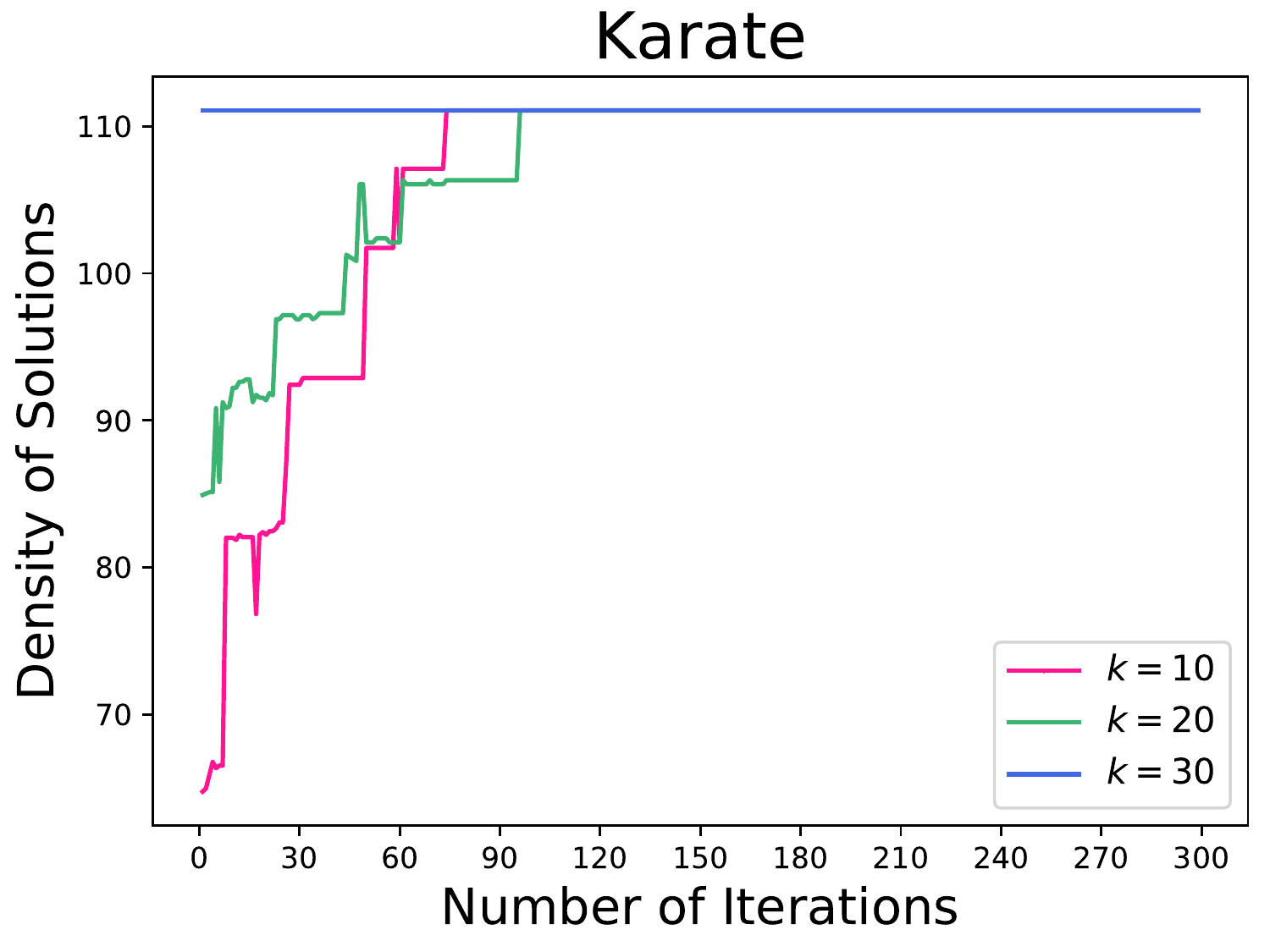}
  }
 \subfigure{
   \includegraphics[width=0.27\textwidth]{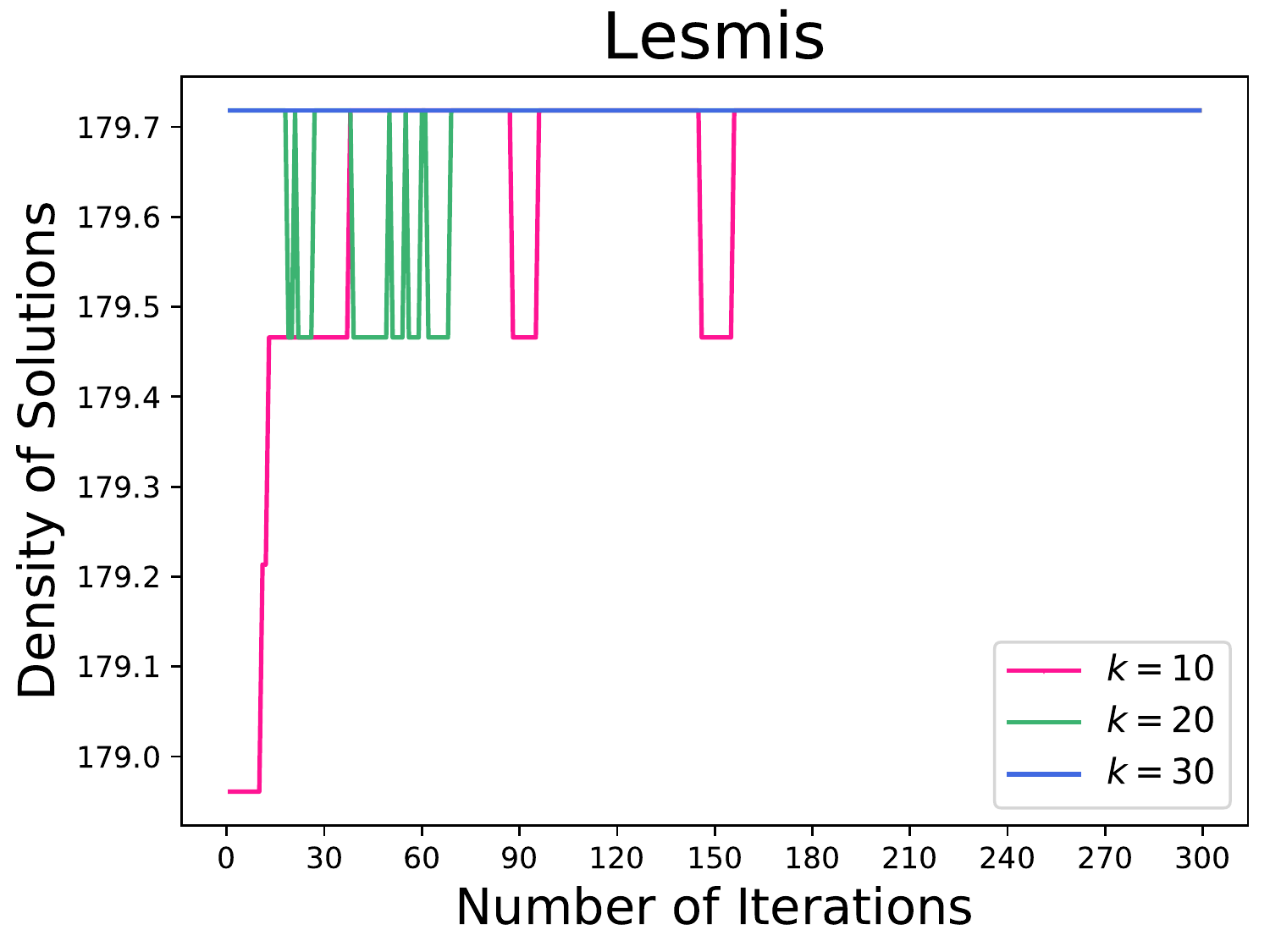}
  }
 \subfigure{
   \includegraphics[width=0.27\textwidth]{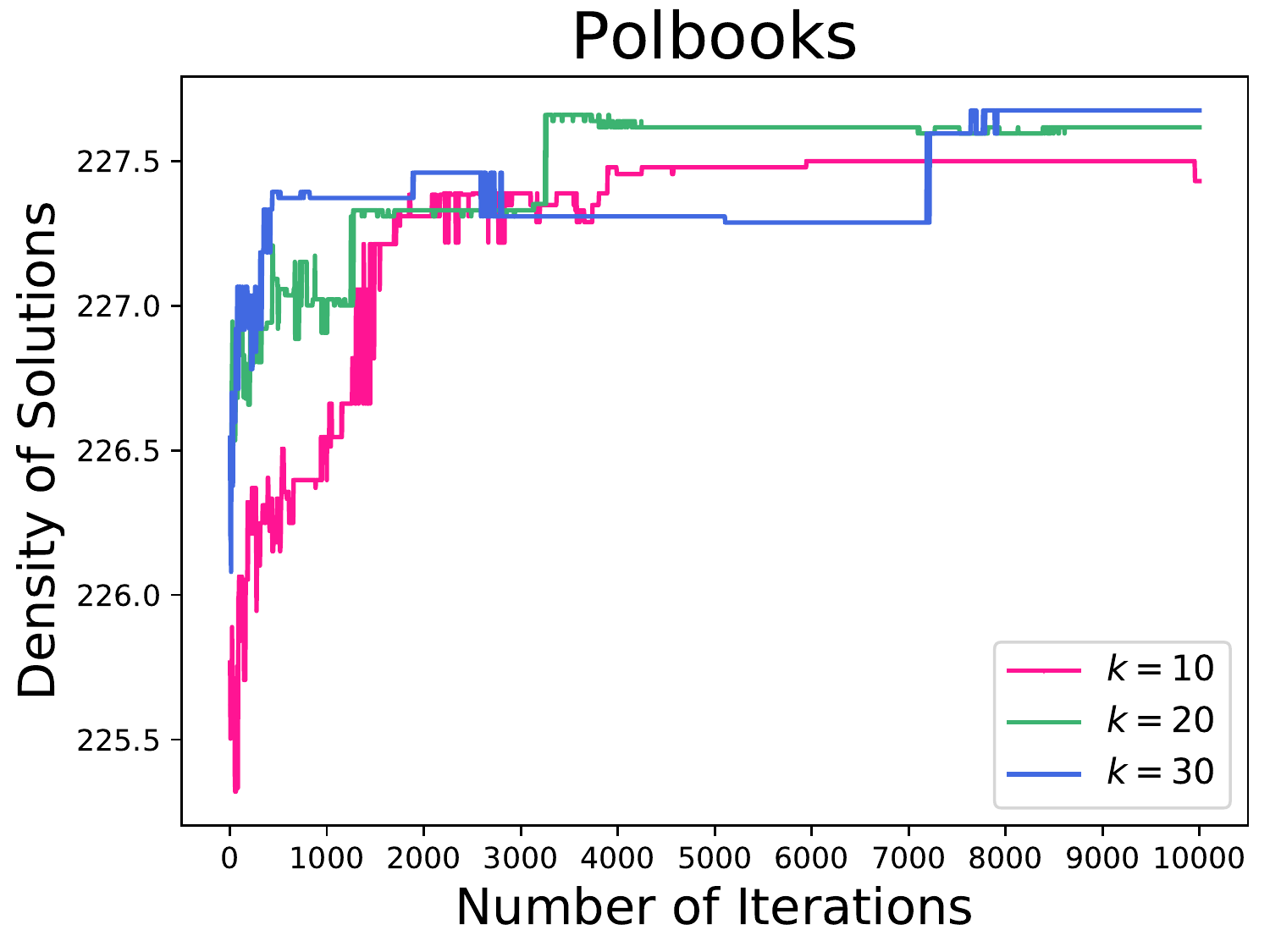}
  }\\
 \subfigure{
   \includegraphics[width=0.27\textwidth]{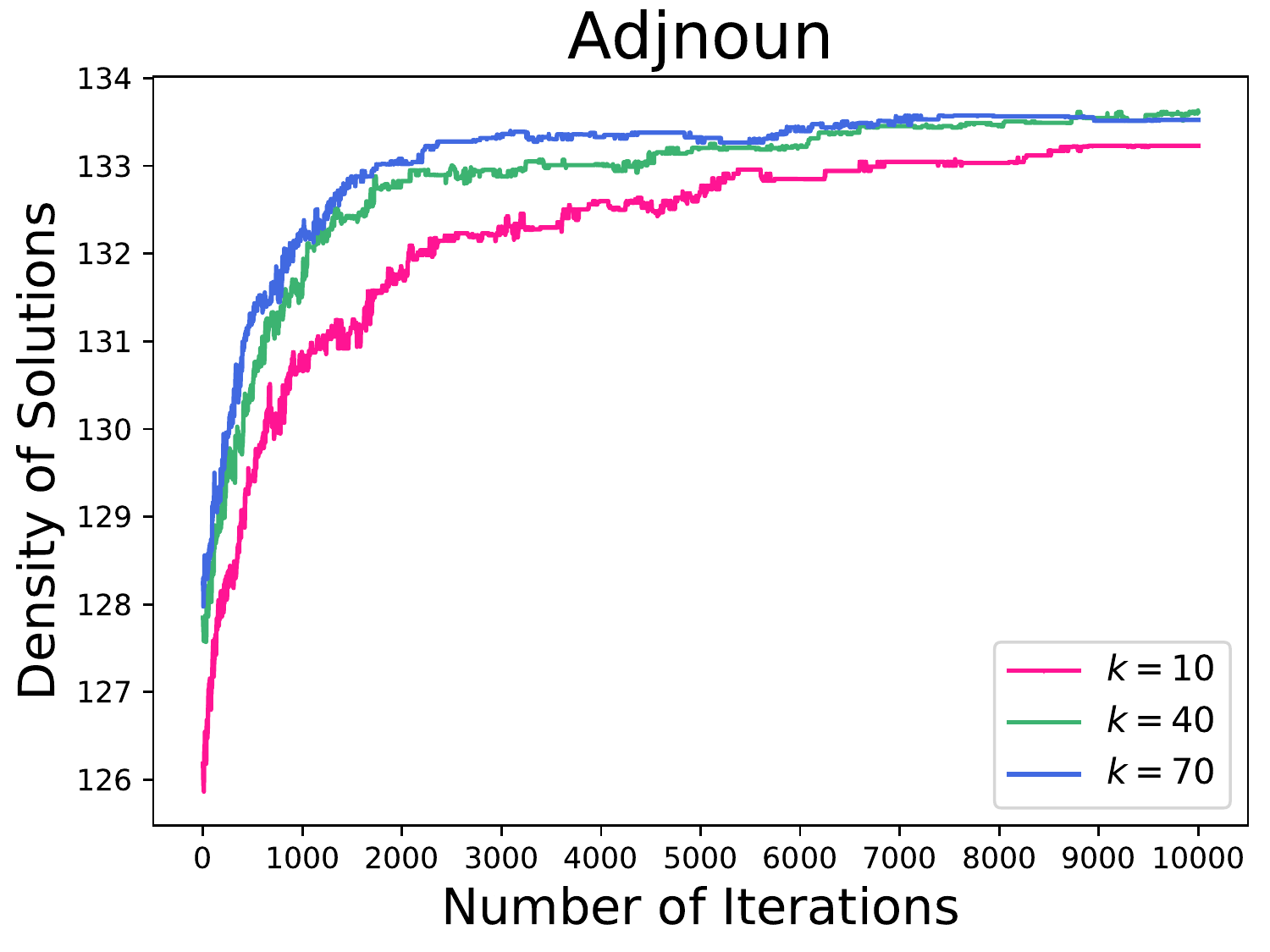}
  }
 \subfigure{
   \includegraphics[width=0.27\textwidth]{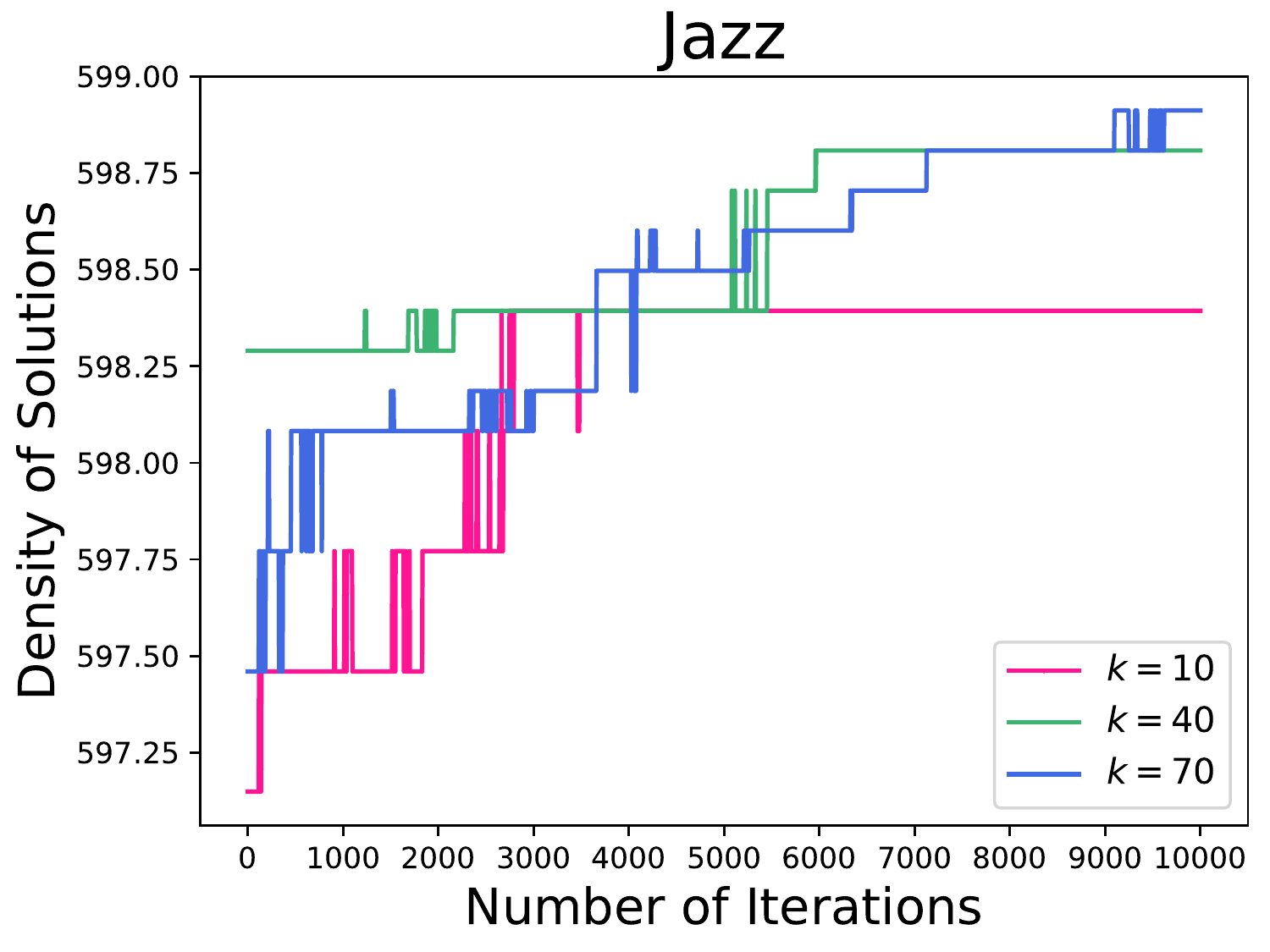}
  }
 \subfigure{
   \includegraphics[width=0.27\textwidth]{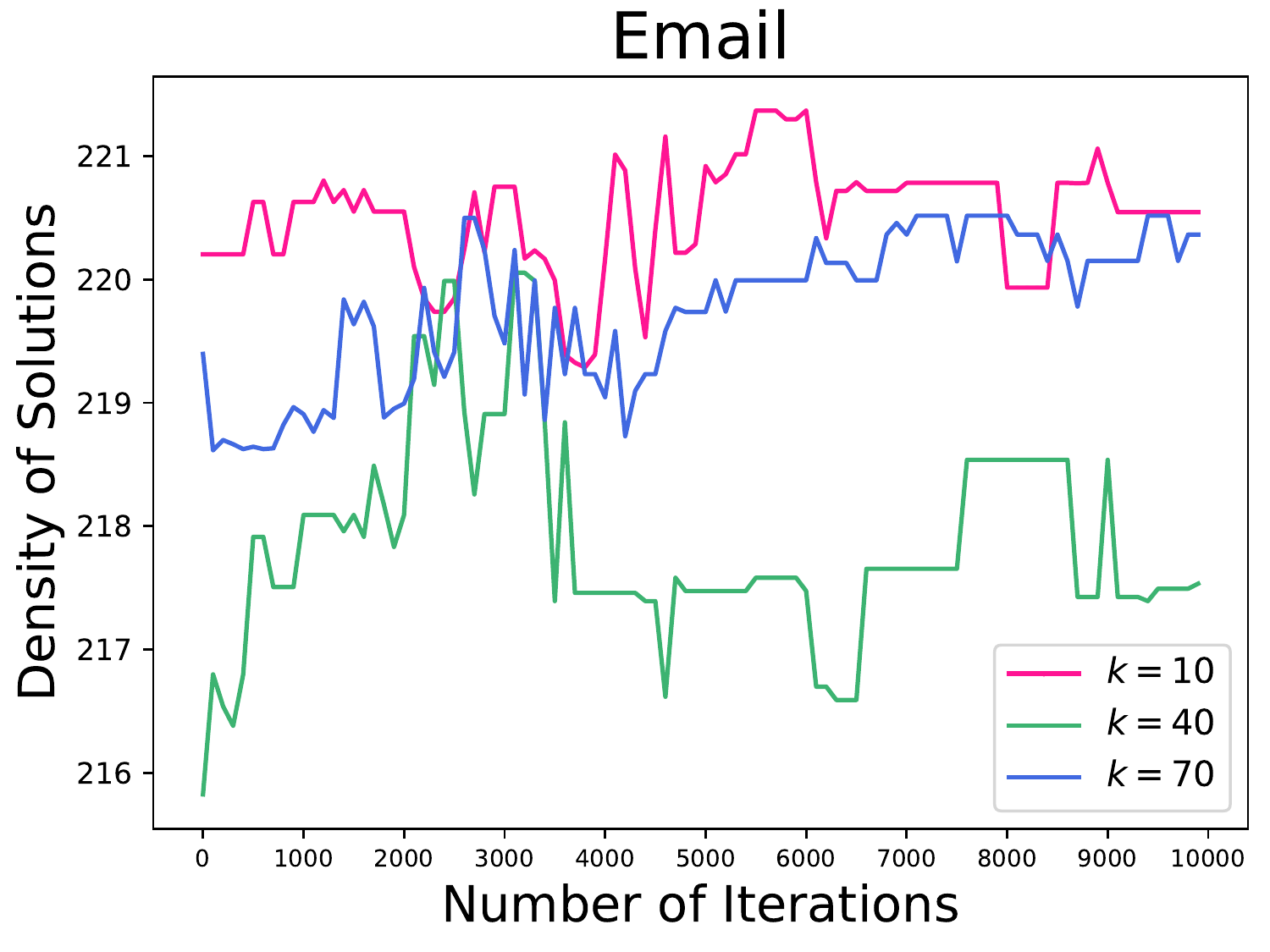}
  }
  \caption{Results for the behavior of our proposed algorithm with respect to the number of iterations. Each point is an average over 10 runs of the algorithm.}\label{fig:solution}
  \end{center}
\end{figure*}

\begin{figure*}[t!]
\begin{center}
 \subfigure{
   \includegraphics[width=0.27\textwidth]{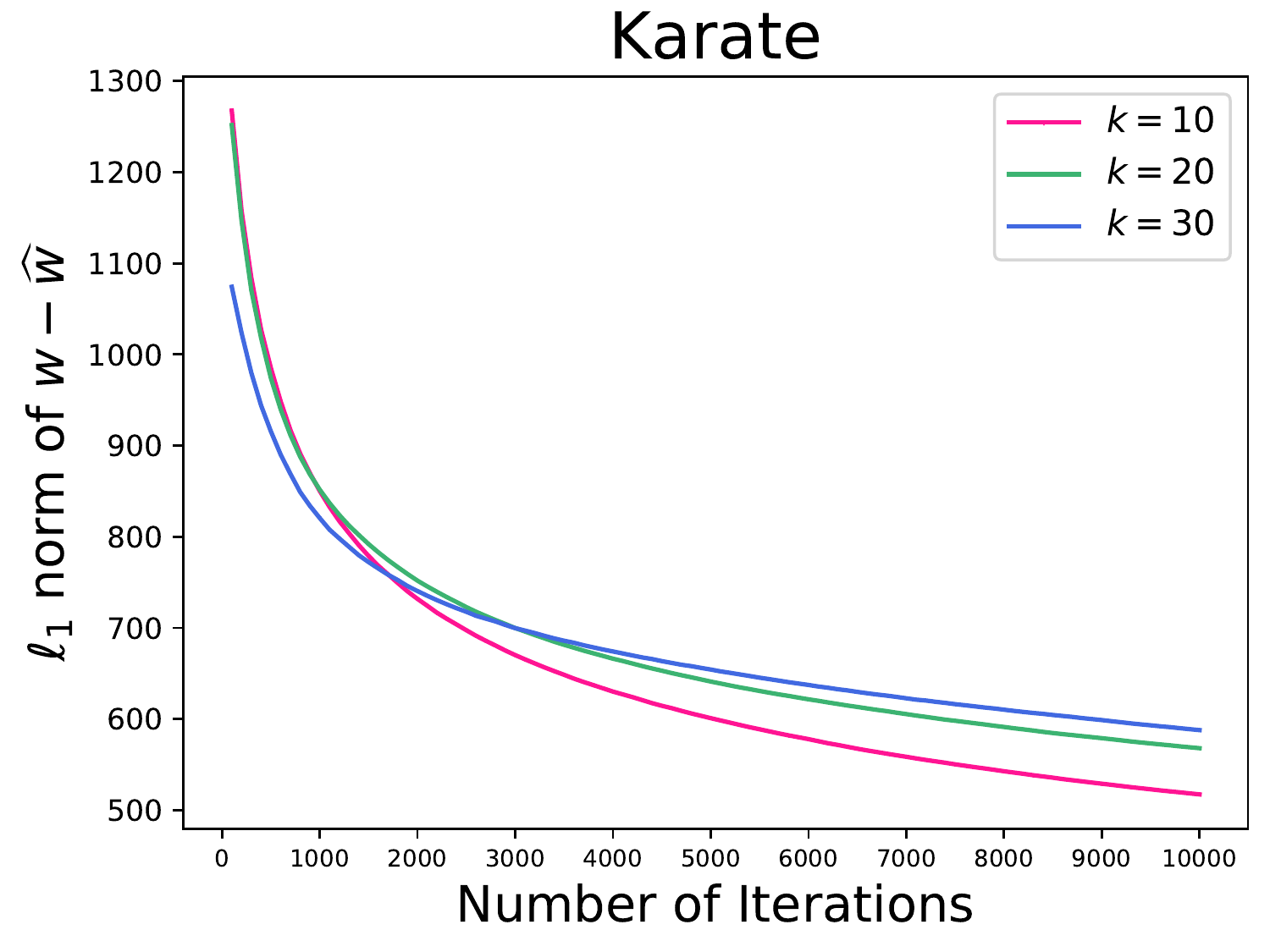}
  }
 \subfigure{
   \includegraphics[width=0.27\textwidth]{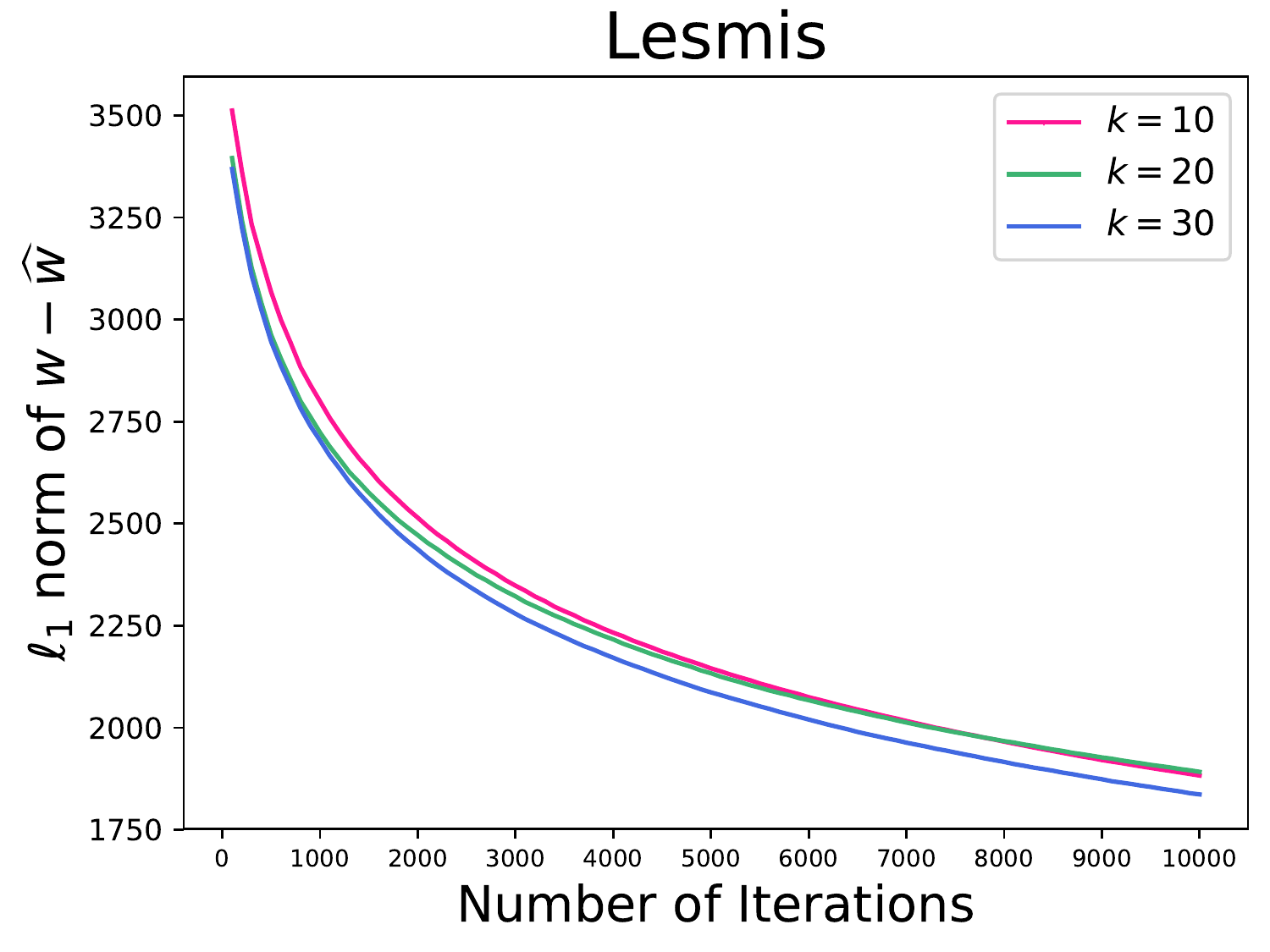}
  }
 \subfigure{
   \includegraphics[width=0.27\textwidth]{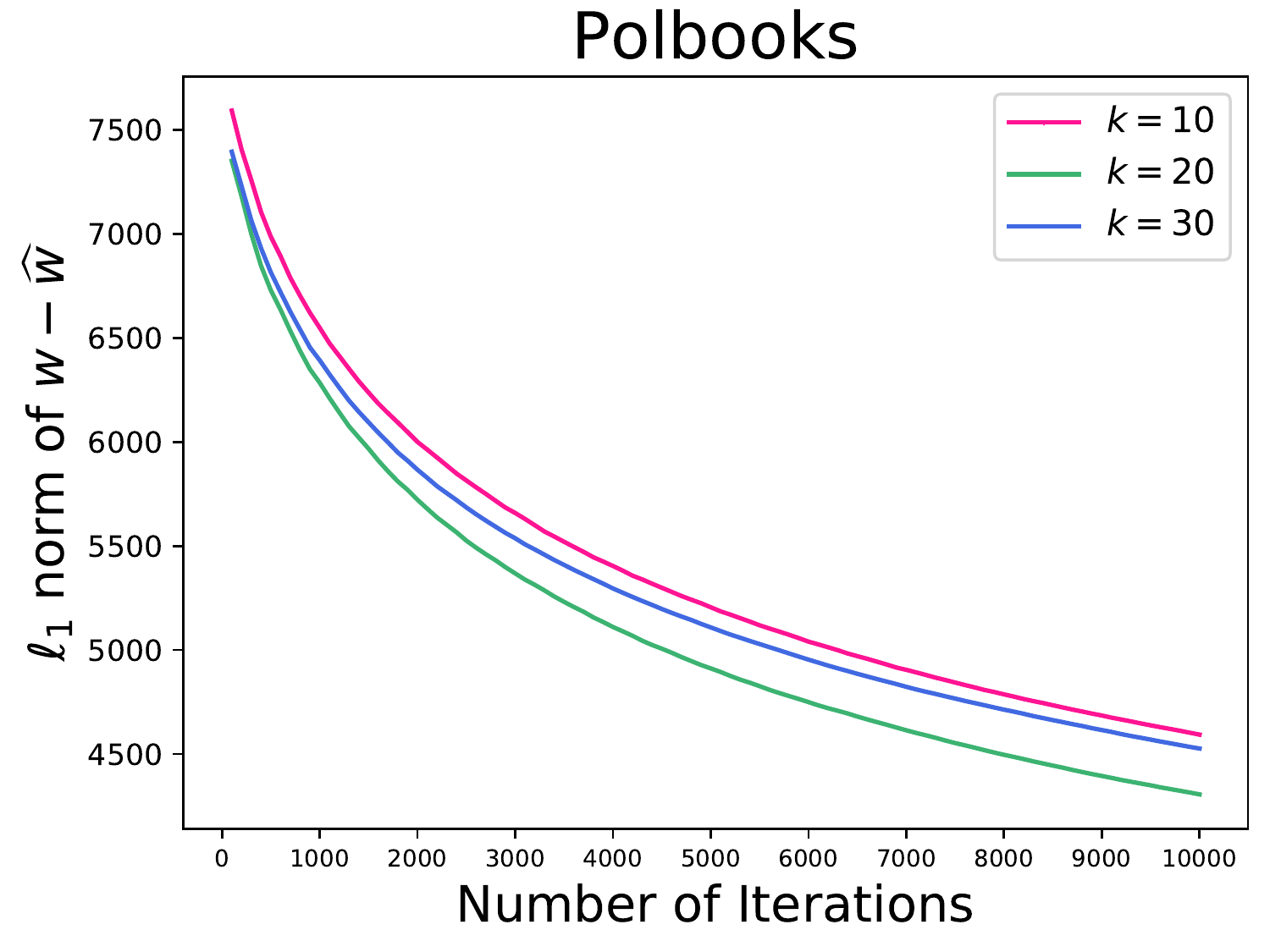}
  }\\
 \subfigure{
   \includegraphics[width=0.27\textwidth]{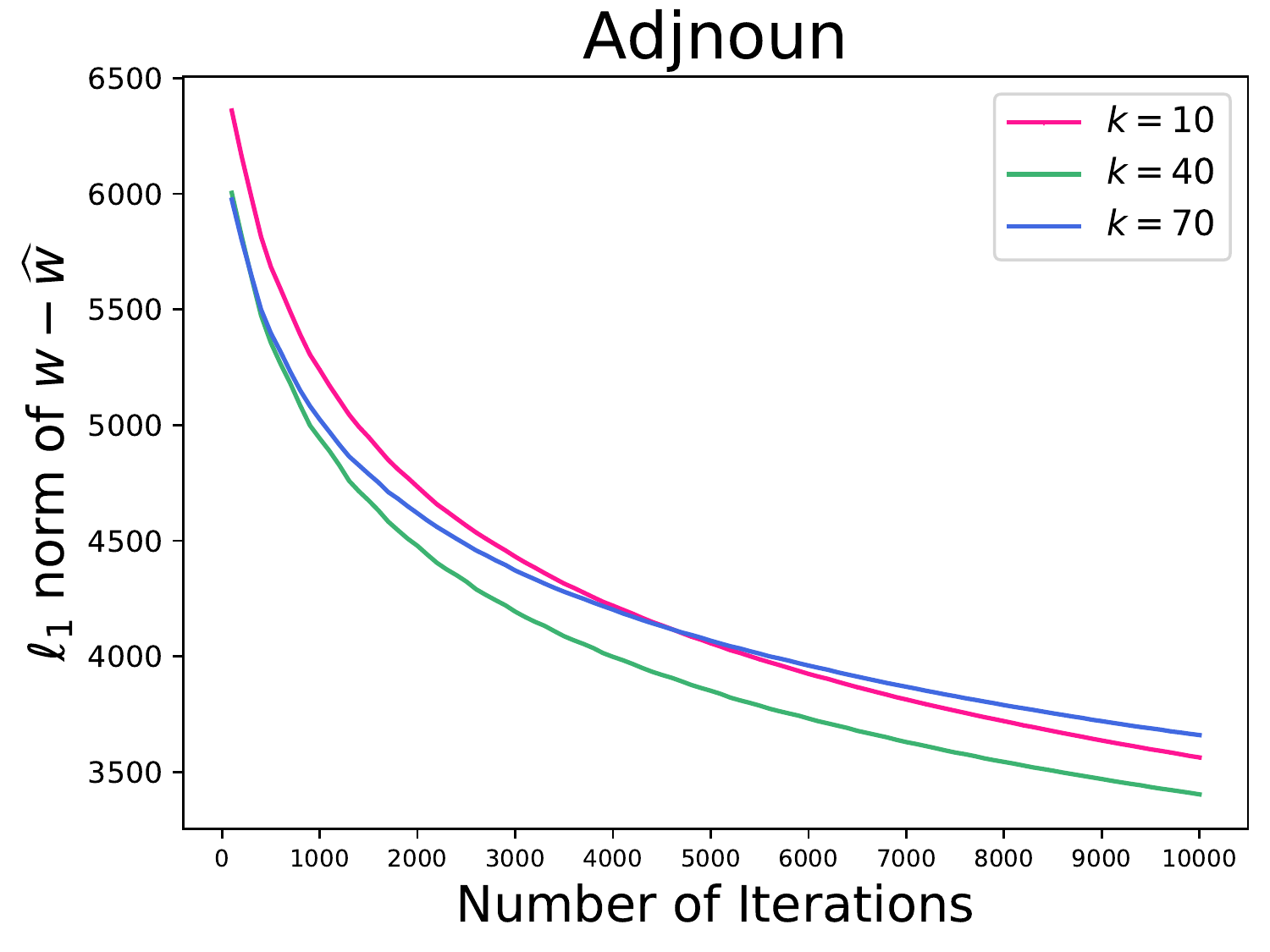}
  }
 \subfigure{
   \includegraphics[width=0.27\textwidth]{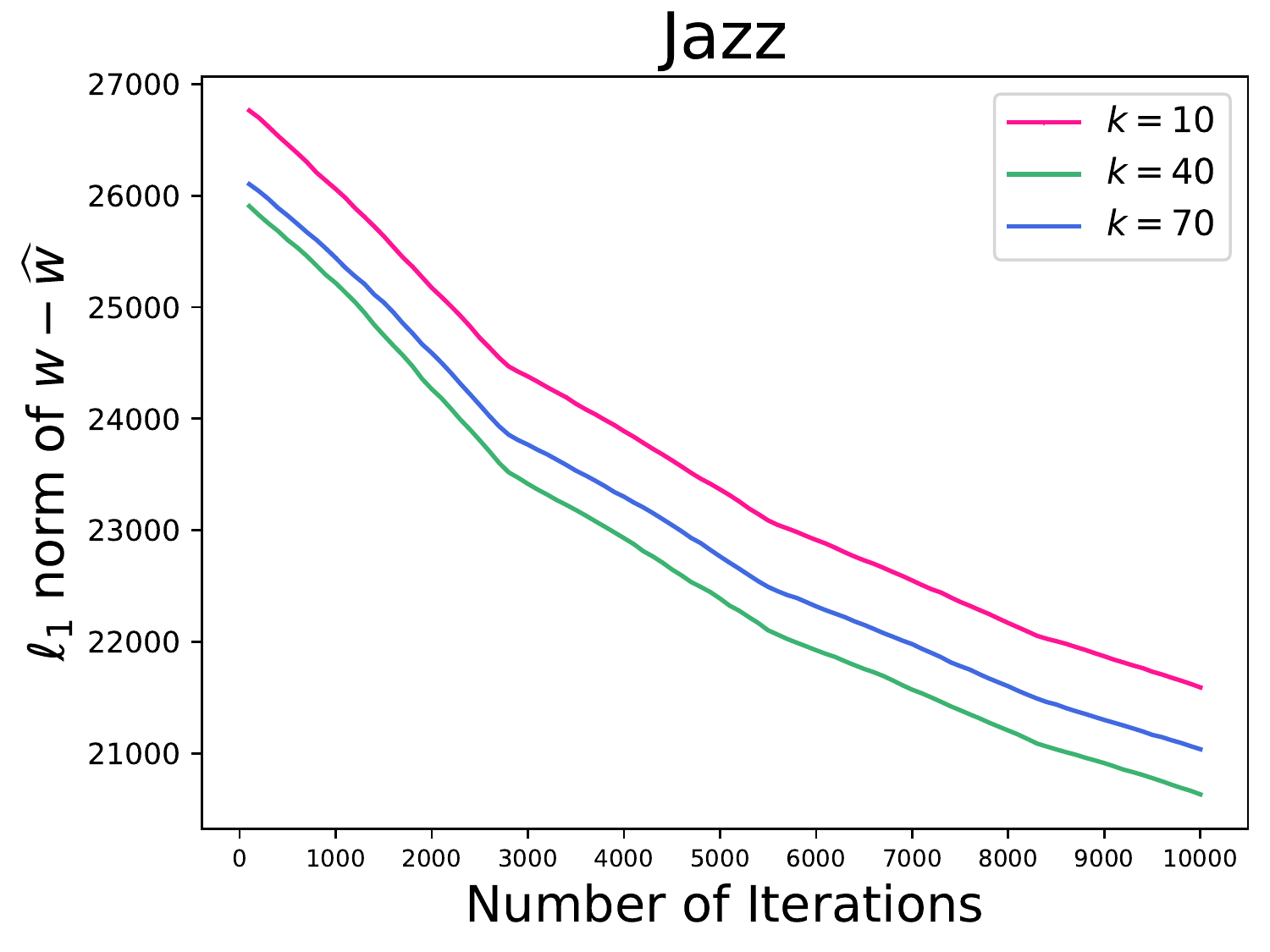}
  }2
 \subfigure{
   \includegraphics[width=0.27\textwidth]{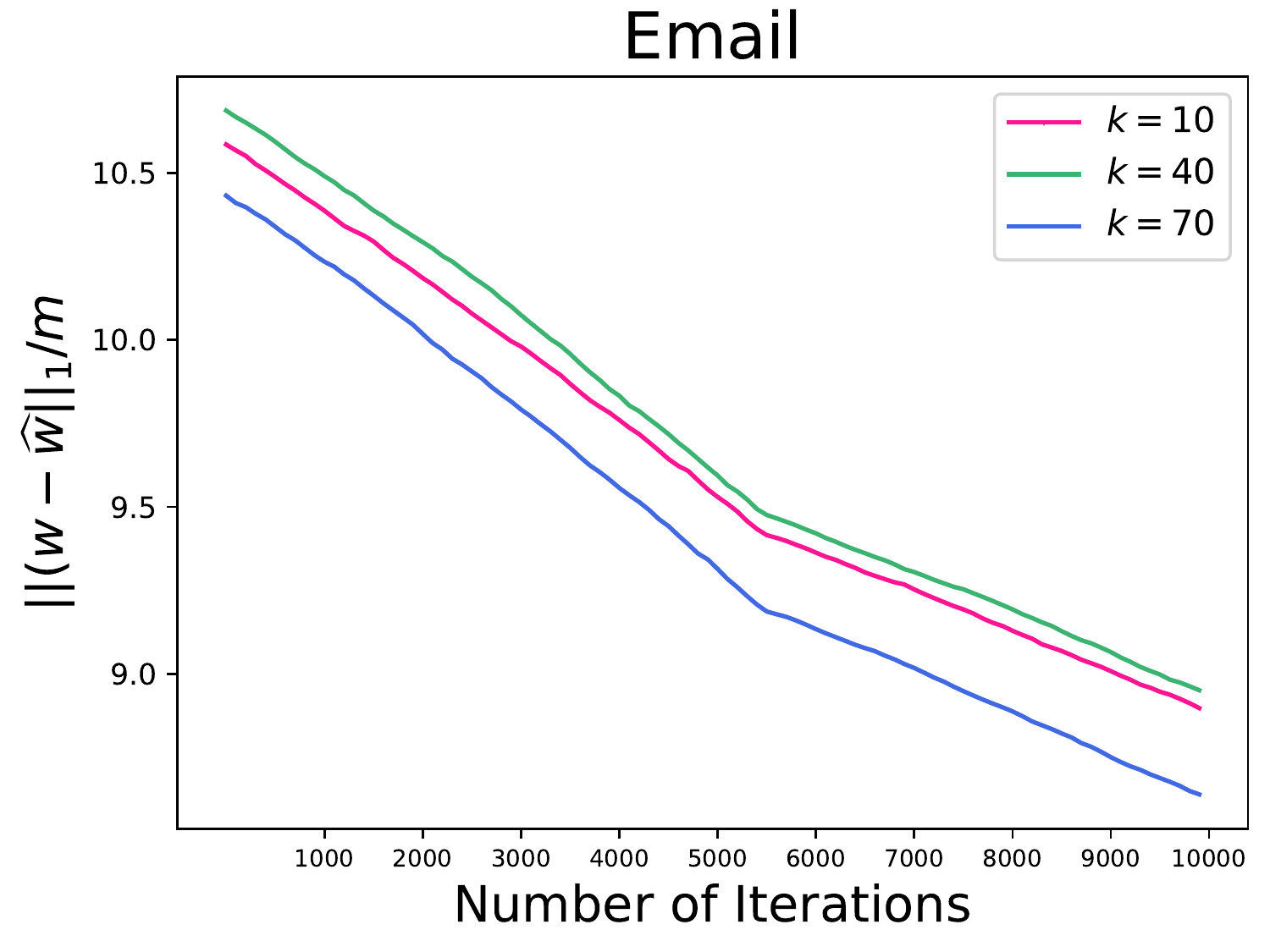}
  }
  \caption{Results for the estimation of the expected weight $w$. Each point is an average over 10 runs of the algorithm.}\label{fig:norm}
  \end{center}
\end{figure*}

\section{Details of experiments for \textsf{DS-Lin}}~\label{apx:exp}
\subsection{Behavior of \textsf{DS-Lin}}

We first analyze the behavior of our proposed algorithm with respect to the number of iterations. 
In the previous section, we confirmed that the solution obtained after 10,000 iterations is almost densest in terms of unknown $w$. 
A natural question here is how the density of solutions approaches to such a sufficiently large value. 
In other words, does our algorithm is sensitive to the choice of the number of iterations? 
In this section, we answer these questions by conducting the following experiments. 
We terminate the while-loop of our algorithm once the number of iterations exceeds $0, 100, 200, \dots$, 10,000, 
and follow the density values of solutions in terms of $w$. 
For each instance, we again run our algorithm for ten times, and report the average value. 

The results are shown in Figure~\ref{fig:solution}. 
As can be seen, as the number of iterations increases, the density value converges to the sufficiently large value (close to the optimum). 
Although the density value sometimes drops down, the decrease is quite small. 

\subsection{Estimation of the expected weight}
We next explain the reason why our proposed algorithm \textsf{DS-Lin} performs fairly well. 
To this end, we focus on the quality of the estimated edge weight obtained by the algorithm. 
We measure the quality of the estimated edge weight $\widehat{w}_t$ by comparing with the expected weight $w$; 
specifically, we compute $\|w -\widehat{w}_t\|_1/m$. 
The experimental setup is exactly the same as that in the previous section. 

The results are depicted in Figure~\ref{fig:norm}. 
As can be seen, as the number of iterations increases, $\widehat{w}_t$ converges to the true edge weight $w$. 
It is very likely that the high performance of our algorithm is derived from the high-quality estimation of the expected edge weight $w$.

\begin{algorithm}[t]
\caption{Robust optimization with oracle intervals (\textsf{R-Oracle})}\label{alg:roracle}
	\SetKwInOut{Input}{Input}
	\SetKwInOut{Output}{Output}
	\Input{ Graph $G=(V,E)$, oracle intervals $W=\times_{e \in E}[l_e, r_e]$ where $l_e=\min\{w_e-1,0 \}$ and $r_e=w_e+1$, sampling oracle, $\gamma \in (0,1)$, and $\varepsilon>0$}
	\Output{($S_{\rm out})$ }
	
    \For{each $e \in E$}{

    \If{$l_e=r_e$ }{
    $l^{\rm out}_e \leftarrow l_e, r^{\rm out}_e \leftarrow r_e$;
    }
    \Else{
    $S^*_{w^{-}} \leftarrow $ Output of Charikar's LP-based exact algorithm for $G(V,E, w^{-})$;
    
   $t_e \leftarrow \left\lceil  \frac{m(r_e-l_e)^2 \ln{\frac{2m}{\gamma}} }{ \varepsilon^2 f_{w^{-}}(S^*_{w^{-}})^2   } \right\rceil $;

    Sample $e$ for $t_e$ times;
    
    $\hat{p}_e \leftarrow \hat{X}_e(t_e)$;
    
    $\delta \leftarrow \frac{\varepsilon f_{w^{-}}(S^*_{w^{-}})}{\sqrt{2m}} $;
    
    $l_e^{\rm out} \leftarrow \max \{ l_e, \hat{p}_e-\delta \}$ and $ r^{\rm out}_e \leftarrow \min \{ r_e, \hat{p}_e +\delta \}$;
    }
    }

    $W_{\rm out} \leftarrow \times_{e \in E}[l^{\rm out}_e, r^{\rm out}_e]$;

    $S_{\rm out} \leftarrow$ Output of Charikar's LP-based exact algorithm for $G(V,E, l^{-}_{\rm out})$;
        
    \Return{
    $S_{\rm out}$;
    }
\end{algorithm}

\section{ Details of experiments for \textsf{DS-SR}}\label{apx:dssr}

\begin{figure*}[t]
\begin{center}
\subfigure{
  \includegraphics[width=0.31\textwidth]{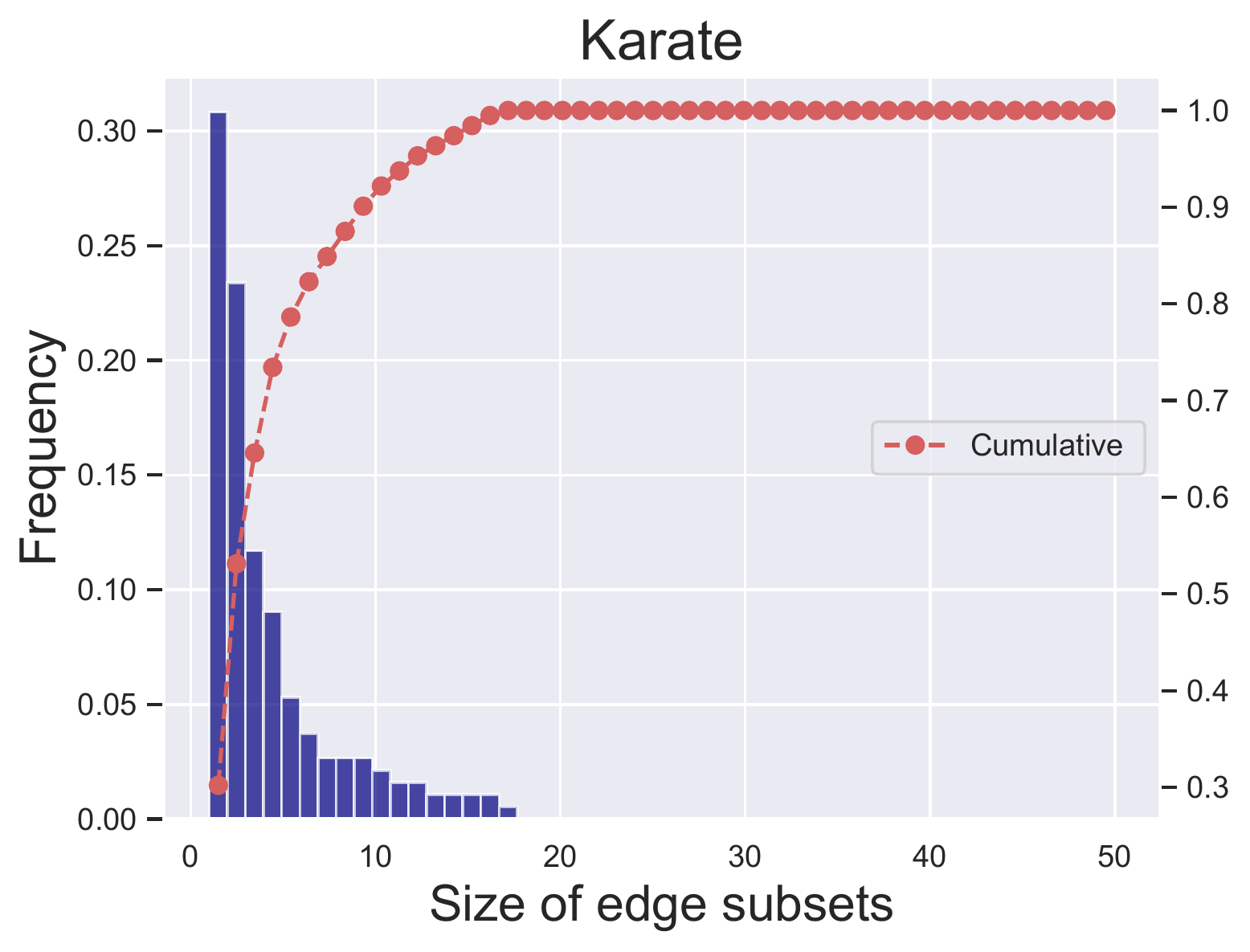}
  }
 \subfigure{
  \includegraphics[width=0.31\textwidth]{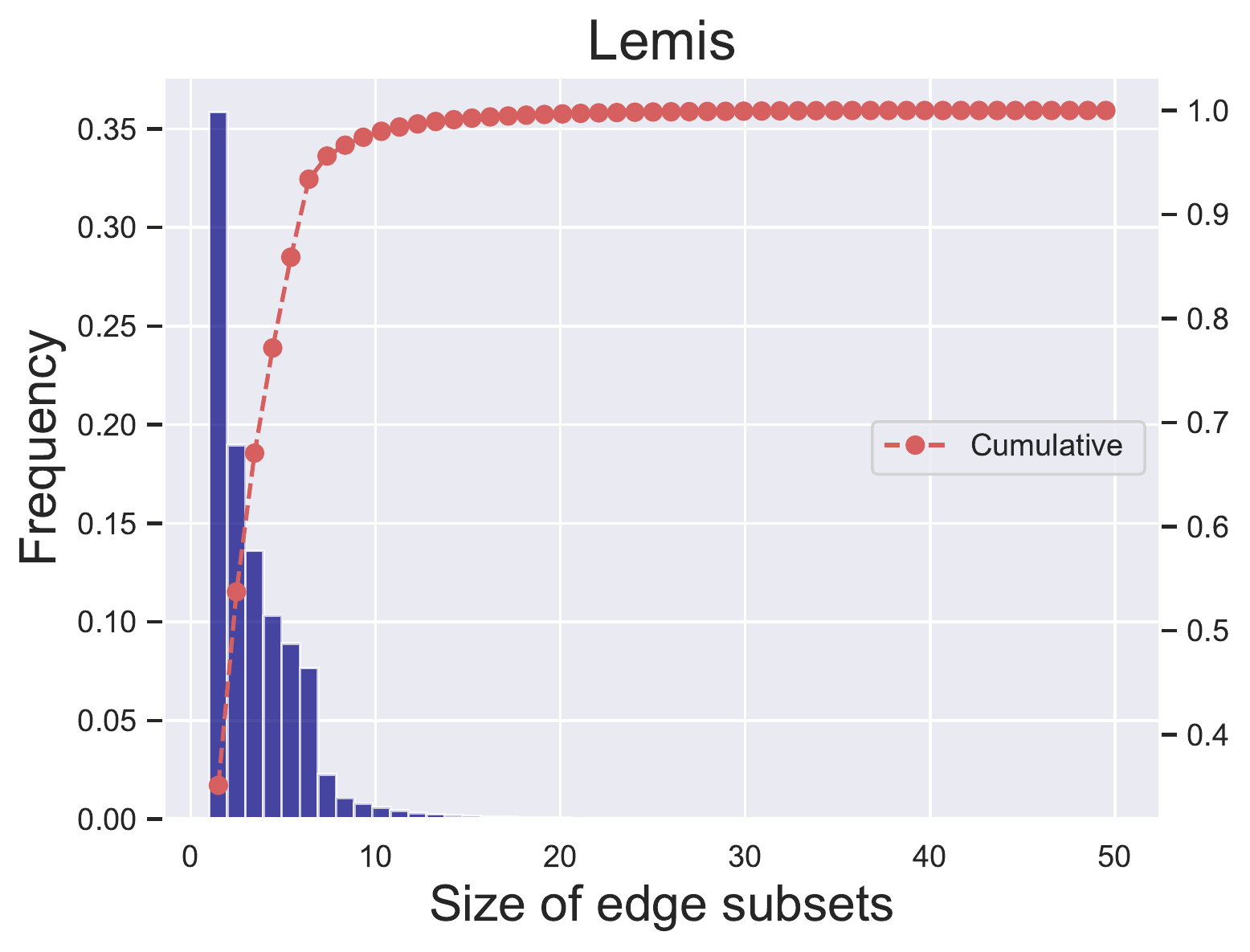}
  }\\
 \subfigure{
  \includegraphics[width=0.31\textwidth]{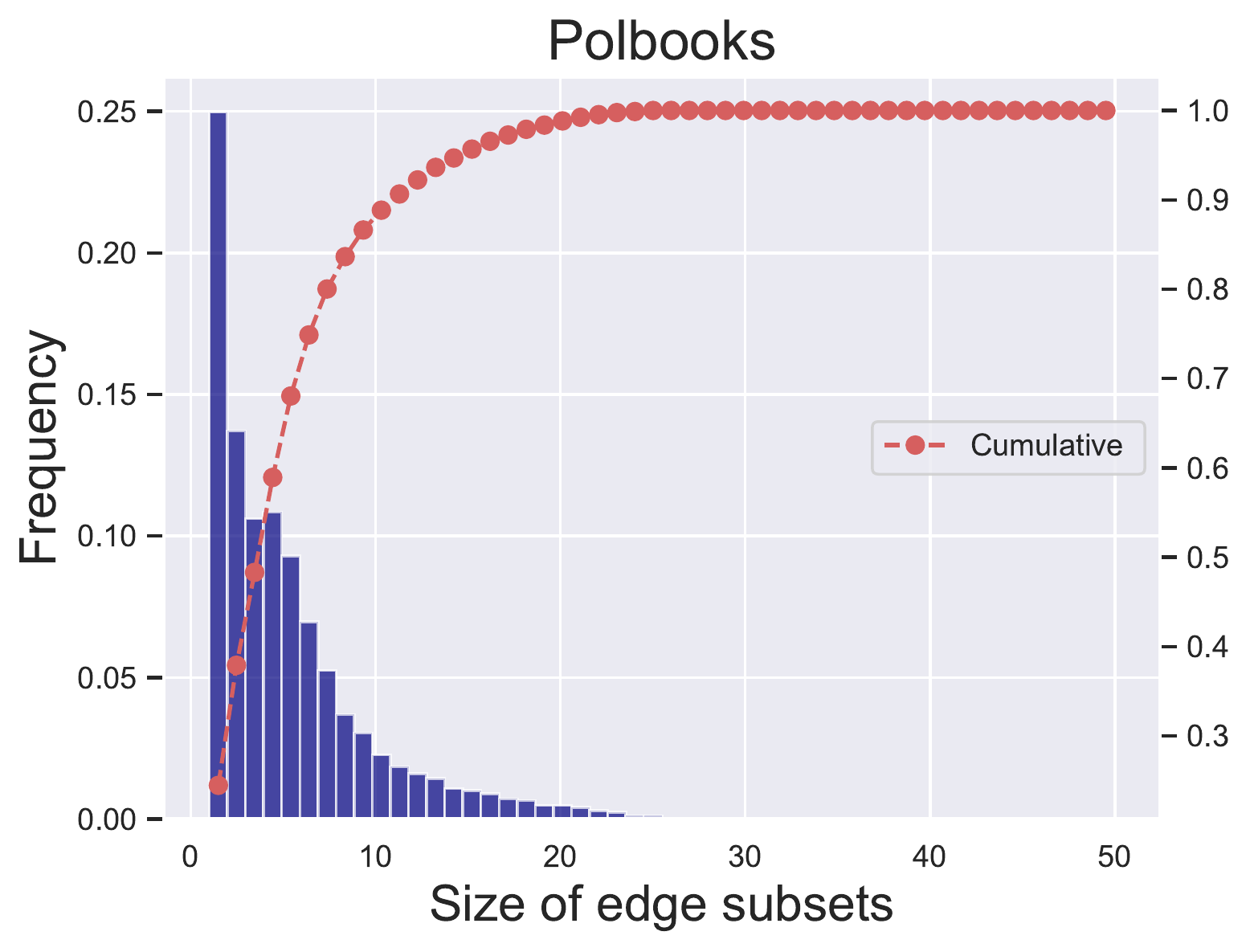}
  }
 \subfigure{
  \includegraphics[width=0.31\textwidth]{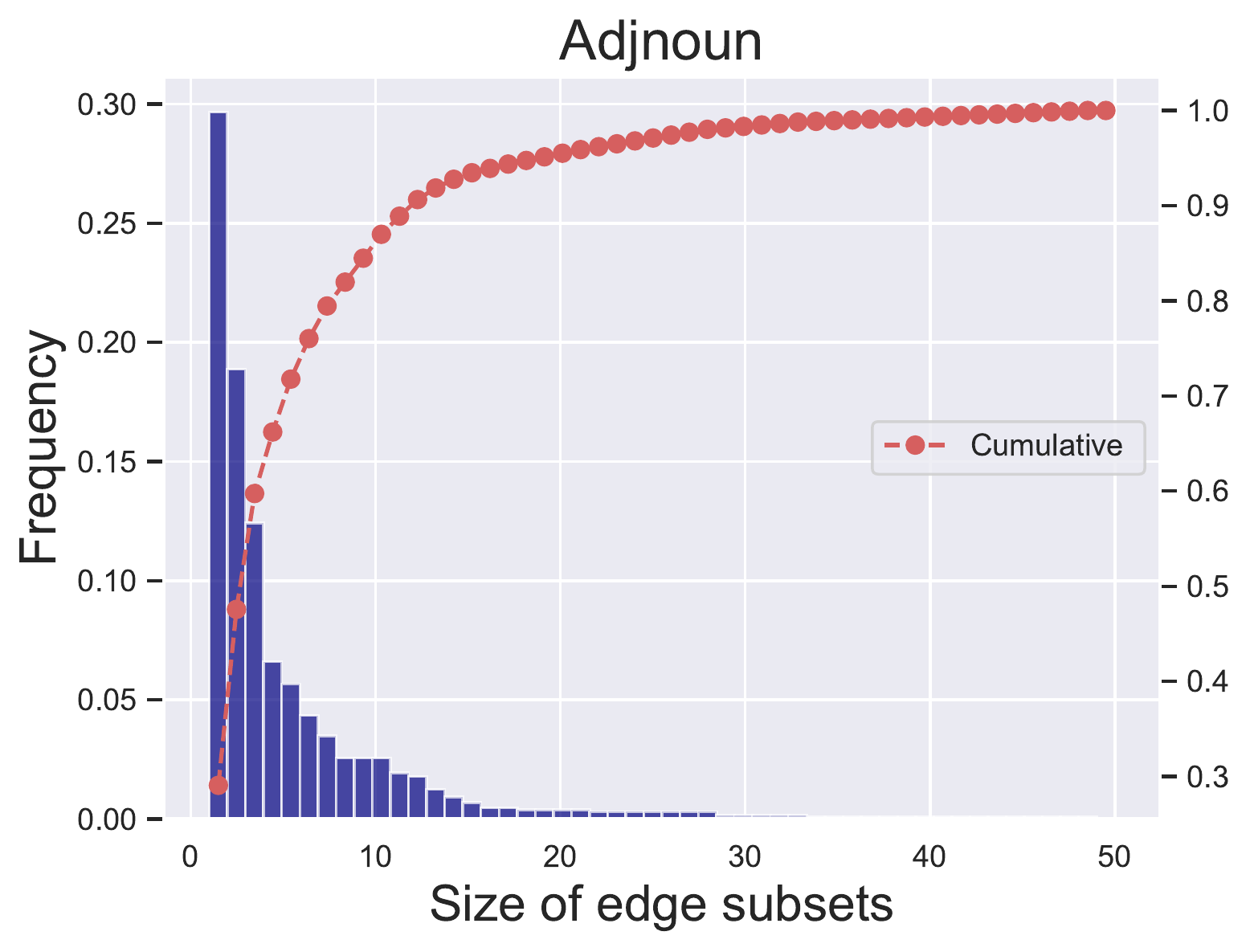}
  }\\
 \subfigure{
  \includegraphics[width=0.31\textwidth]{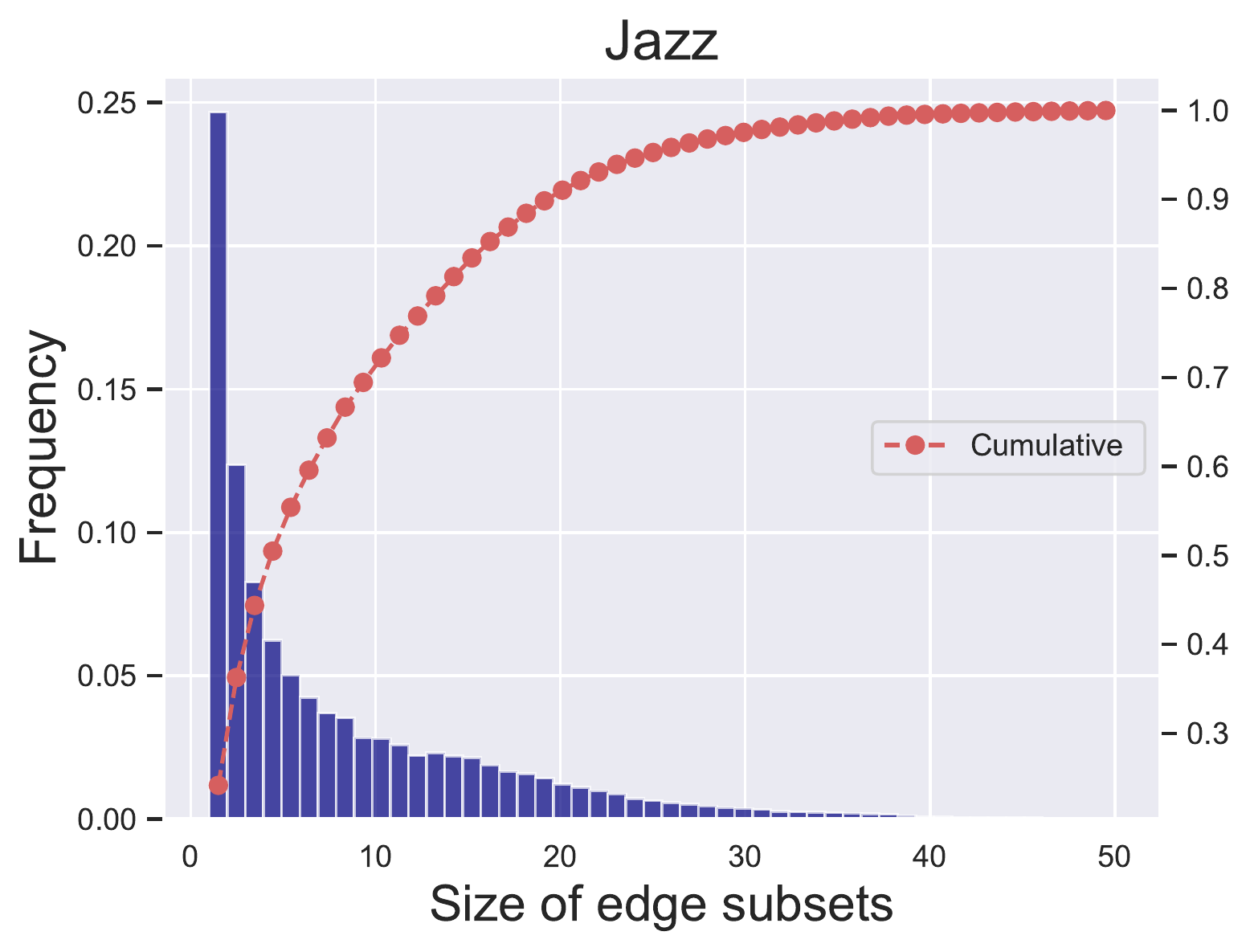}
  }
 \subfigure{
  \includegraphics[width=0.31\textwidth]{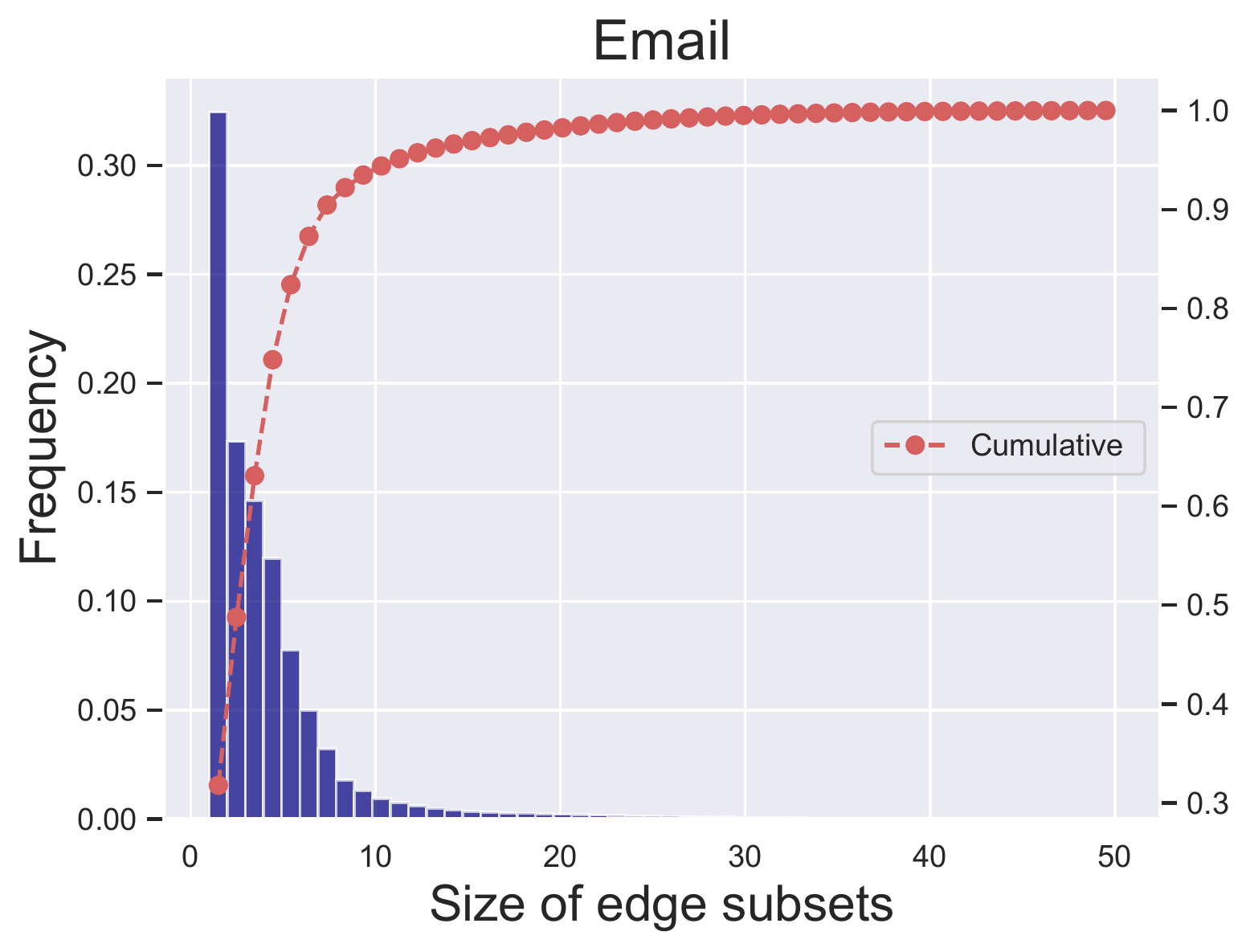}
  }\\
  \subfigure{
  \includegraphics[width=0.31\textwidth]{fig/sample/polblogs.pdf}
  }
 \subfigure{
  \includegraphics[width=0.31\textwidth]{fig/sample/email-Eu-core.pdf}
  }\\
 \subfigure{
  \includegraphics[width=0.31\textwidth]{fig/sample/facebook_combined.pdf}
  }
 \subfigure{
  \includegraphics[width=0.31\textwidth]{fig/sample/Wiki-Vote.pdf}
  }
  
  \caption{Fraction of the size of queried edge subsets in \textsf{DS-SR} (cumulative) over 100 runs.}\label{fig:cumlative_all}
  \end{center}
\end{figure*}

\subsection{Description of \textsf{R-Oracle} }
We describe the entire procedure of \textsf{R-Oracle} in Algorithm~\ref{alg:roracle}.
This algorithm employs the robust optimization model proposed by~\citet{Miyauchi_Takeda_18}.
Their robust optimization model takes intervals of edge weights as its input.
We generate the intervals $W= \times_{e \in E}[l_e, r_e]$ based on unknown edge weight $w$, i.e., $l_e=\min\{w_e-1,0 \}$ and $r_e=w_e+1$. 
Algorithm~\ref{alg:roracle} first obtains the optimal solution $S^*_{w^{-}}$ in terms of extreme edge weight $w^{-}=(l_e)_{e \in E}$ and computes the value of $f_{w^{-}}(S^*_{w^{-}})$.
Then, for each single edge $e \in E$,
the algorithm calls the sampling oracle for an appropriate number of times and obtains the empirical mean. Using the empirical means, the algorithm constructs intervals $W_{\rm out} \leftarrow \times_{e \in E}[l^{\rm out}_e, r^{\rm out}_e]$,
and computes a densest subgraph $S_{\rm out}$ on $G$ with $w^{-}_{\rm out}=(l^{\rm out}_e)_{e \in E}$.

\subsection{ The number of samples for single edges in \textsf{DS-SR}}
We report experimental results on the size of queried edge subsets in \textsf{DS-SR} (cumulative) over 100 runs for all instances in Figure~\ref{fig:cumlative_all}.

\end{onecolumn}

\end{document}